\documentclass[11pt]{amsart}
\usepackage[T1]{fontenc}
\usepackage[utf8]{luainputenc}

\usepackage{array}
\usepackage{graphicx}
\usepackage{verbatim}
\usepackage{fullpage}
\usepackage[english]{babel}
\usepackage{yfonts}
\usepackage{appendix}
\usepackage{hyperref}
\usepackage{xcolor}

\usepackage{tikz}
\usepackage{pgfplots}
\usepackage{pgfkeys}
\usepackage{pgfmath}

\usepackage{listings}
\usepackage{verbatim}
\usepackage{framed}
\usepackage{fancyvrb}

\definecolor{vgreen}{RGB}{104,180,104}
\definecolor{vblue}{RGB}{49,49,255}
\definecolor{vorange}{RGB}{255,143,102}

\lstdefinestyle{verilog-style}
{
  language=Verilog,
  basicstyle=\small\ttfamily,
  keywordstyle=\color{vblue},
  identifierstyle=\color{black},
  commentstyle=\color{vgreen},
  numbers=left,
  numberstyle=\tiny\color{black},
  numbersep=10pt,
  tabsize=8,
  moredelim=*[s][\colorIndex]{[}{]},
  literate=*{:}{:}1
}
\usepackage{color}
\definecolor{deepblue}{rgb}{0,0,0.5}
\definecolor{deepred}{rgb}{0.6,0,0}
\definecolor{deepgreen}{rgb}{0,0.5,0}

\lstdefinestyle{python-style}
{
  language=Python,
  basicstyle=\small\ttfamily,
  commentstyle=\color{deepgreen},
  otherkeywords={self},             
  keywordstyle=\color{deepblue},
  emph={MyClass,__init__},          
  emphstyle=\color{deepred},        
  stringstyle=\color{deepgreen},
  frame=tb,                         
  showstringspaces=false            %
}

\makeatletter
\newcommand*\@lbracket{[}
\newcommand*\@rbracket{]}
\newcommand*\@colon{:}
\newcommand*\colorIndex{%
  \edef\@temp{\the\lst@token}%
  \ifx\@temp\@lbracket \color{black}%
  \else\ifx\@temp\@rbracket \color{black}%
  \else\ifx\@temp\@colon \color{black}%
  \else \color{vorange}%
  \fi\fi\fi
}
\makeatother

\usepackage{trace}

\newtheorem{theorem}{Theorem}
\newtheorem{lemma}{Lemma}
\newtheorem{example}{Example}

\theoremstyle{definition}
\newtheorem{exercise}{Exercise}

\usepackage{algorithm}
\usepackage{algpseudocode}

\makeatletter
\newcommand*{\algrule}[1][\algorithmicindent]{\makebox[#1][l]{\hspace*{.5em}\vrule height .75\baselineskip depth .25\baselineskip}}%

\newcount\ALG@printindent@tempcnta
\def\ALG@printindent{%
  \ifnum \theALG@nested>0
  \ifx\ALG@text\ALG@x@notext
  \addvspace{-3pt}
  \else
  \unskip
  \ALG@printindent@tempcnta=1
  \loop
  \algrule[\csname ALG@ind@\the\ALG@printindent@tempcnta\endcsname]%
  \advance \ALG@printindent@tempcnta 1
  \ifnum \ALG@printindent@tempcnta<\numexpr\theALG@nested+1\relax
  \repeat
  \fi
  \fi
}%
\usepackage{etoolbox}
\patchcmd{\ALG@doentity}{\noindent\hskip\ALG@tlm}{\ALG@printindent}{}{\errmessage{failed to patch}}
\makeatother
\makeatother

\algnewcommand{\True}{\textbf{true}\space}
\algnewcommand{\False}{\textbf{false}\space}
\algnewcommand{\And}{\textbf{and}}
\algnewcommand{\Or}{\textbf{ or }}
\algnewcommand{\Break}{\textbf{break}}%
\algnewcommand{\Continue}{\textbf{continue}}%
\algnewcommand{\algorithmicgoto}{\textbf{go to}}%
\algnewcommand{\Goto}[1]{\algorithmicgoto~\ref{#1}}%
\algnewcommand{\Input}[1]{\textbf{Input:} #1}
\algnewcommand{\Output}[1]{\textbf{Output:} #1}
\algnewcommand{\Divider}{{\smallskip\color{red}\hrule\smallskip}}

\algdef{SE}[VARIABLES]{Variables}{EndVariables}{\algorithmicvariables}%
{\algorithmicend\ \algorithmicvariables}
\algnewcommand{\algorithmicvariables}{\textbf{global}}

\algnewcommand{\Global}{\textbf{global}\space}

\title{Deductron --- A Recurrent Neural Network}
\author{Marek Rychlik\\
  University of Arizona\\
  Department of Mathematics, 617 N Santa Rita Rd, P.O. Box 210089\\
  Tucson, AZ 85721-0089, USA\\}

\date{\today}



\gdef\NumStates{11}

\newenvironment{inact}{%
  \catcode`\_11%
}{%
}

\begin{document}
\begin{abstract}
  The current paper is a study in Recurrent Neural Networks (RNN),
  motivated by the lack of examples simple enough so that they can be
  thoroughly understood theoretically, but complex enough to be
  realistic.  We constructed an example of structured data, motivated
  by problems from image-to-text conversion (OCR), which requires
  long-term memory to decode. Our data is a simple writing system,
  encoding characters 'X' and 'O' as their upper halves, which is
  possible due to symmetry of the two characters.  The characters can
  be connected, as in some languages using cursive, such as Arabic
  (abjad).  The string 'XOOXXO' may be encoded as
  '${\vee}{\wedge}\kern-1.5pt{\wedge}{\vee}\kern-1.5pt{\vee}{\wedge}$'.
  It is clear that seeing a sequence fragment
  '$|\kern-1.8pt{\wedge}\kern-1.5pt{\wedge}\kern-1.5pt{\wedge}\kern-1.5pt{\wedge}\kern-1.5pt{\wedge}\kern-1.8pt|$'
  of any length does not allow us to decode the sequence as '\ldots
  XXX\ldots' or '\ldots OOO \ldots' due to inherent ambiguity, thus
  requiring long-term memory. Subsequently we constructed an RNN
  capable of decoding sequences like this example. Rather than by
  training, we constructed our RNN ``by inspection'', i.e. we guessed
  its weights. This involved a sequence of steps. We wrote a
  conventional program which decodes the sequences as the example
  above.  Subsequently, we interpreted the program as a neural network
  (the only example of this kind known to us). Finally, we generalized
  this neural network to discover a new RNN architecture whose
  instance is our handcrafted RNN. It turns out to be a three-layer
  network, where the middle layer is capable of performing simple
  logical inferences; thus the name ``deductron''. It is demonstrated
  that it is possible to train our network by simulated
  annealing. Also, known variants of stochastic gradient descent (SGD)
  methods are shown to work.
\end{abstract}
\maketitle
\section{Introduction}
Recurrent Neural Networks (RNN) have gained significant attention in
recent years due to their success in many areas, including speech
recognition and image-to-text conversion, Optical Character
Recognition, or OCR. These are systems which respond to sequential
inputs, such as time series.  With skillfull implementation they have
the ability to react to the stimuli \emph{in real time}, which is at
the root of their applications to building intelligent systems.  The
classes of RNN which memorize and forget a certain amount of
information are especially interesting.

Yet, it is hard to find in literature examples of data which can be
easily understood, and which demonstrably require remembering and
forgetting information to operate correctly. In this paper we will
provide such an example of data, define the related machine learning
problem and solve it using typical machine learning tools. Our
analysis will be rigorous whenever possible, reflecting our
mathematical and computer science point of view. Thus, we will
constantly pivot between three subjects (math, computer science and
connectionist artificial intelligence) hopefully providing an
insightful study, which can be continued in various directions by the
reader. We also included a number of exercises varying in the degree
of difficulty which should make reading more fun.

In the current paper specifically, we are interested in explaining the
need for long-term memory, in addition to short-term memory.  In the
last 20 years LSTM (Long-Short Memory) RNNs have been applied to a
variety of problems with artificial intelligence flavor, in
particular, speech-to-text conversion and optical character
recognition \cite{lstm,journals/neco/GersSC00}.  We find that typical
examples used to illustrate LSTM are too complex to understand how the
network performs its task:
\begin{enumerate}
\item Why is there a need for long and short term memory in specific
  problems?
\item What are the necessary ingredients of the neural network
  architecture that can utilize long-short term memory?
\end{enumerate}

In order to have a suitable example of data, we constructed a simple
(artificial) writing system (we will call it the W-language, or ``wave
language''), encoding characters 'X' and 'O' as their upper halves,
i.e. $\vee$ and $\wedge$ (this is possible due to reflectional
symmetry of 'X' and 'O' and no other two Latin characters would do).
The characters \textbf{can be connected}, as in some languages. Thus
'XOOXXO' is encoded in our alphabet as
'${\vee}{\wedge}\kern-1.5pt{\wedge}{\vee}\kern-1.5pt{\vee}{\wedge}$'.
Hence, the written text looks like a sequence of waves, with one
restriction: a wave that starts at the bottom (top), must end at the
bottom (top).

Let us explain the fact that decoding sequences of characters requires
long-term memory.  It is clear that seeing a sequence fragment
'$|\kern-1.8pt{\wedge}\kern-1.5pt{\wedge}\kern-1.5pt{\wedge}\kern-1.5pt{\wedge}\kern-1.5pt{\wedge}\kern-1.8pt|$'
of any length does not allow us to decode the sequence as '\ldots
XXX\ldots' or '\ldots OOO \ldots' due to inherent ambiguity.  Thus, it
is necessary to remember the beginning of the ``wave'' (bottom or top)
to resolve this ambiguity. Hence the need for memory; in fact, we need
to remember what was written arbitrarily long time ago in order to
determine whether a given sequence should be decoded as a sequence of
'X' or as a sequence of 'O'.

Having invented our (artificial) writing system, we construct an RNN
(in some ways similar to LSTM) capable of decoding sequences like the
examples provided above, with 100\% accuracy in the absence of
errors. In the presence of errors, the accuracy should gracefully drop
off, demonstrating robustness; this will not be pursued in the current
paper.

What we will focus on is a construction of the RNN network in an
unusual, and hopefully enlightning way. Rather than proposing a
network architecture in a ``blue skies research'' fashion (or looking
at prior work), we wrote a conventional program which decodes the
sequences as the example above, operating on a binary image
representation, with vertical resolution of three pixels.
Subsequently, we re-interpreted the program as a neural network, and
thus obtained a neural network ``by inspection'' (the only non-trivial
example of this sort we are aware of). We then generalized this neural
network to discover \textbf{a new RNN architecture} whose instance is
our handcrafted RNN.  It turns out to be a three-layer network, where
\textbf{the middle layer is capable of performing simple logical
  inferences}; thus the name \textbf{deductron} will be used for our
newly discovered architecture.

The next stage of our study is to pursue machine learning, using the
new RNN architecture. We considered two methods of machine learning:
\begin{enumerate}
\item simulated annealing;
\item Stochastic Gradient Descent (SGD). 
\end{enumerate}

In particular, we developed a training algorithm for the new
architecture, by minimizing a standard cost function (also called the
\emph{loss function} in the machine learning community) with simulated
annealing.  The training algorithm was demonstrated to find a set of
weights and biases of the neural network which yields a decoder
solving the decoding problem for the W-language. In some runs, the
decoder is logically equivalent to the manually constructed decoder.
Thus, we proved that our architecture can be trained to write programs
functionally equivalent to hand-coded programs written by a human. It
is possible to learn a decoding algorithm from a single sample of
length $30$ (encoding the string 'XOOXXO').

We also applied a different method of training the deductron called
\textbf{back-propagation through time} (BPTT) and known to succeed in
training other RNNs. This, and other back-propagation based algorithms
require computing gradients of complicated functions, necessitating
application of the Chain Rule over complex dependency graphs.  Modern
tools perform the gradient calculation automatically. One such tool is
Tensorflow \cite{tensorflow2015-whitepaper}.  We implemented machine
learning using Tensorflow and some programming in Python.  We took
advantage of the SGD implementation in Tensorflow.  In particular, we
used the \emph{Adam optimizer} \cite{2014arXiv1412.6980K}.  Using
standard steps, we demonstrated that the decoder for the W-language
can be constructed by learning from a small sample of valid sequences
(of length $\approx 500$).

Both simulated annealing and BPTT methods worked with relative ease
when applied to our problem of decoding the W-language.

\section{The W-language  and writing system}
In the current paper we study a toy example of a system for sending
messages like:
\begin{verbatim}
...XOOXXO...
\end{verbatim}
The message is thus expressed as a string in alphabet consisting of
letters 'X' and 'O'. However, we assume that the message is
transcribed by a human or a human-like system, by writing it on paper,
and scanning it to a digital image, e.g.  like in
Figure~\ref{fig:sample1}.

\begin{figure}[htb]
  \begin{center}
    \includegraphics[width=5in]{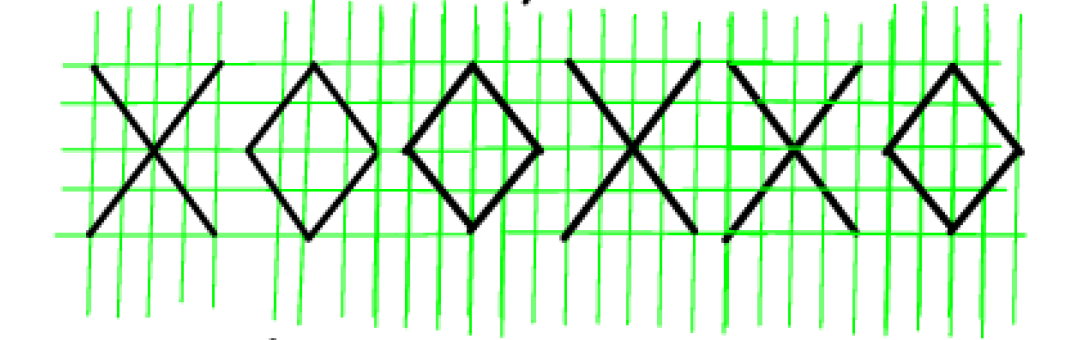}
  \end{center}
  \caption{Handwritten sample representing string
    'XOOXXO'.\label{fig:sample1}}
\end{figure}
Letters 'X' and 'O' were chosen because they are symmetric with
respect to reflections along the horizontal axis. We assume that the
receiver of the message sees only the upper half of the message, which
could look like Figure~\ref{fig:sample2}. Thus, our effective alphabet
is
\[ \mathcal{A} = \left\{\vee, \wedge\right\}. \]
However, when rendering the messages in this alphabet, we \textbf{may}
connect the consecutive characters, as in various script-based
languages, i.e. we write in \emph{cursive}.
\begin{figure}[htb]
  \begin{center}
    \includegraphics[width=5in]{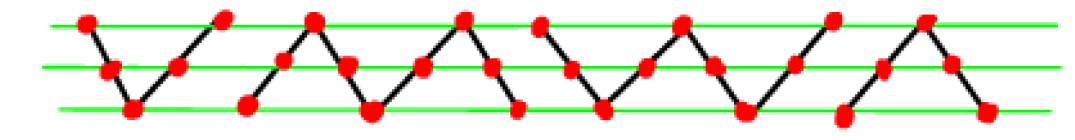}
  \end{center}
  \caption{The top portion of the handwritten sample representing
    string 'XOOXXO'.\label{fig:sample2}}
\end{figure}
The message is also subject to errors of various kinds, resulting in
something like Figure~\ref{fig:sample3}. More severe errors could be,
for instance, random bit flips, i.e. the input message could be
subjected to the \emph{binary symmetric channel}
\cite{citeulike:141092}.

\begin{figure}[htb]
  \begin{center}
    \includegraphics[width=5in]{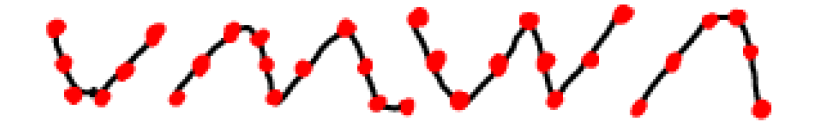}
  \end{center}
  \caption{The top portion of the handwritten sample representing
    string 'XOOXXO', with some errors.\label{fig:sample3}}
\end{figure}

For the purpose of constructing a minimalistic example still possessing the
features of the motivating example, we think of digitized representations
of the messages, which are 5 pixels tall. Thus the ``top'' of the message
is only 3 pixels tall, and it consists of a sequence of vectors
representing the columns of the image. Let
$0=(0,0,0)$, $e_1=(1,0,0)$, $e_2=(0,1,0)$ nd $e_3=(0,0,1)$ be the
vectors which may occur if we are precisely observing the rules
of calligraphy of our messages, as illustrated by Figure~\ref{fig:sample1}.
Our sample message 'XOOXXO' is thus represented by the sequence
of vectors:
\[  0, e_3, e_2, e_1, e_2, e_3, e_1, e_2, e_3, e_2, e_1, e_2, e_3, e_2,e_1,
  e_3, e_2, e_1, e_2, e_3, e_2, e_1, e_2, e_1, e_1, e_2, e_3, e_1, e_2, e_3, e_2, e_1. \]
We could consider ``errors'' obtained by inserting extra $0$ vectors between $e_1$ and $e_3$
signaling a long break between symbols 'X' and 'O'. We could repeat some vectors.
Generally, the image should consist of a number of ``waves'' and ``breaks''.

We could also represent the image as a matrix of bits, as in
Figure~\ref{fig:binary-image}.
\def\pixels{
  {0,1,0,0,0,1,0,0,1,0,0,0,1,0,0,1,0,0,0,1,0,0,0,1,0,0,1,0,0,0},
  {0,0,1,0,1,0,0,1,0,1,0,1,0,1,0,0,1,0,1,0,1,0,1,0,0,1,0,1,0,0},
  {0,0,0,1,0,0,1,0,0,0,1,0,0,0,1,0,0,1,0,0,0,1,0,0,1,0,0,0,1,0}%
}
As image ($30\times 3$):

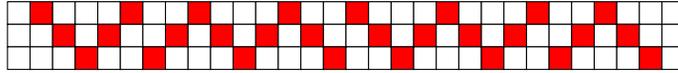
\begin{figure}
  \definecolor{pixel 0}{HTML}{FFFFFF}
  \definecolor{pixel 1}{HTML}{FF0000}
  \begin{center}
    \begin{tikzpicture}[scale=0.3]
      \foreach \line [count=\y] in \pixels {
        \foreach \pix [count=\x] in \line {
          \draw[fill=pixel \pix] (\x,-\y) rectangle +(1,1);
        }
      }
    \end{tikzpicture}
  \end{center}
  \caption{Binary image of the message 'XOOXXO'.\label{fig:binary-image}}
\end{figure}
In Figure~\ref{fig:raw-bits} we represent the image as raw data (a 2D matrix of bits).
\begin{figure}
  \begin{center}
    \begin{tikzpicture}[scale=0.3]
      \foreach \line [count=\y] in \pixels {
        \foreach \pix [count=\x] in \line {
          \node at (\x,-\y) {$\pix$};
        }
      }
    \end{tikzpicture}
  \end{center}
  \caption{Binary image of the message 'XOOXXO' as raw bits.\label{fig:raw-bits}}
\end{figure}
We note that the ``wave'' portion of the pattern may be arbitrarily
long.  However, a picture like Figure~\ref{fig:sample4} cannot be
interpreted as a long sequence '...XXXX...' or '...OOOO...'. We must
go back to the last ``break'' (one of the transitions $0\to e_1$,
$0\to e_3$, $e_2\to e_1$, $e_2\to e_3$ which begins a run of 'O' or 'X'.
\begin{figure}[htb]
  \begin{center}
    \includegraphics[width=3in]{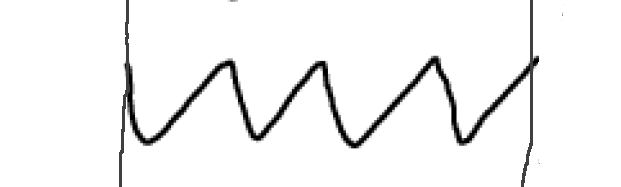}
  \end{center}
  \caption{A wave.\label{fig:sample4}}
\end{figure}
Thus decoding an image like Figure~\ref{fig:sample2} is very similar
to decoding a sequence encoded using Run Length Encoding (RLE), in
which we code runs of characters 'X' and 'O'. The transition tells us
whether we are starting an 'X' ($*\to e_3$, where $*$ denotes $0$ or
$e_2$) or 'O' ($*\to e_1$).

The image consists of a certain number of complete waves possibly
separated with breaks.  A properly constructed complete wave begins
and ends in the same vector, either $e_1$ or $e_3$. It can be divided
into rising and falling spans.  For example, a rising span would be a
sequence $e_1,e_2,e_2,e_3,e_3$.  That is, the non-zero coordinate of
the vector moves upwards.  A break is simply a run of $0$
vectors. Such a run must be preceeded and followed by $e_1$ or $e_3$.
Since the rising and falling spans are of arbitrary length, we must
remember whether we are rising or falling, to validate the sequence,
and to prevent spans like $e_3,e_2,e_2,\ldots,e_2,e_3$ which should
not occur in a valid sequence.  A complete wave starting with $e_3$
must begin with a falling span, and alternate rising and falling spans
afterwards, finally terminating with a rising span. In order to
decode a wave correctly as a sequence of 'X' or 'O', we must remember
whether we are currently rising or falling.

In short, we have to remember two things:
\begin{enumerate}
\item Are we within 'X' or 'O'?
\item Are we rising or falling?
\end{enumerate}
There is some freedom in choosing the moment when to emit a character
'X' or 'O'.  We could do it as soon as we begin a rising or falling
span terminating in the vector which started the wave. Or we can wait
for completion of the span, e.g., when a rising span ends and a
falling span begins, or has a jump $e_3\to e_1$ or $e_3\to 0$ (jump
$e_3\to e_2$ would be an error).

There is a simple graphical model (a topological Markov chain) which
generates all error-free sequences which can be decoded, in
Figure~\ref{fig:topological-markov-chain}. As we can see, the states
of the Markov chain correspond to the vectors $0$, $e_1$, $e_2$ and
$e_3$, except that vector $e_2$ has two corresponding states:
$e_2^\pm$.  The state $e_2^+$ ($e_2^-)$ can only be entered when we
encounter vector $e_2$ on a rising (falling) span. Thus, the state
$e_2^\pm$ is a state that ``remembers'' whether it is on a rising or
falling span. The total number of states is thus $5$.

In computer science and computer engineering the more common term is
\emph{finite state machine} (FSM) or \emph{finite state
  automaton}. This is essentially a Topological Markov Chain with
distinguished initial and final states.  Our Topological Markov Chain
generates complete expressions of the W-language iff they start at
$0$, $e_1$ or $e_3$. Thus initial and final states are these three
states.

\begin{figure}[htb]
  \begin{center}
    \begin{tikzpicture}[node distance=3cm,
      every edge/.style={thick,->,draw},
      every node/.style={circle,draw,thick}
      ]
      \node[circle](e1){$e_1$};
      \node(zero)[circle,below of = e1] {$0$};
      \node(e2p)[circle,left of = zero] {$e_2^+$};
      \node(e2m)[circle,right of = zero]{$e_2^-$};
      \node(e3)[circle,below of = zero]{$e_3$};
      \path (e1) edge[bend right,->]  (zero);
      \path (e1) edge[->] (e2p);
      \path (e1) edge[bend right,->] (e3);
      \path (e3) edge[bend right,->] (e1);        
      \path (e2p) edge[->] (e3);
      \path (zero) edge[->] (e1);
      \path (zero) edge[->] (e3);
      \path (e3) edge[bend right,->] (zero);
      \path (e3) edge[->] (e2m);
      \path (e2m) edge[->] (e1);            
      \path (e1) edge[loop above,->] (e1);
      \path (e3) edge[loop below,->] (e3);
      \path (zero) edge[loop right,->] (zero);
      \path (e2p) edge[loop left,->] (e2p);
      \path (e2m) edge[loop right,->] (e2m);                

    \end{tikzpicture}
  \end{center}
  \caption{The topological Markov chain which can be used to generate
    training data for our network. A valid transition sequence should
    start and end on one of the nodes: $e_1$, $e_3$ or $0$.
    Thus, it cannot start or end at $e_2^{\pm}$.
    \label{fig:topological-markov-chain}}
\end{figure}
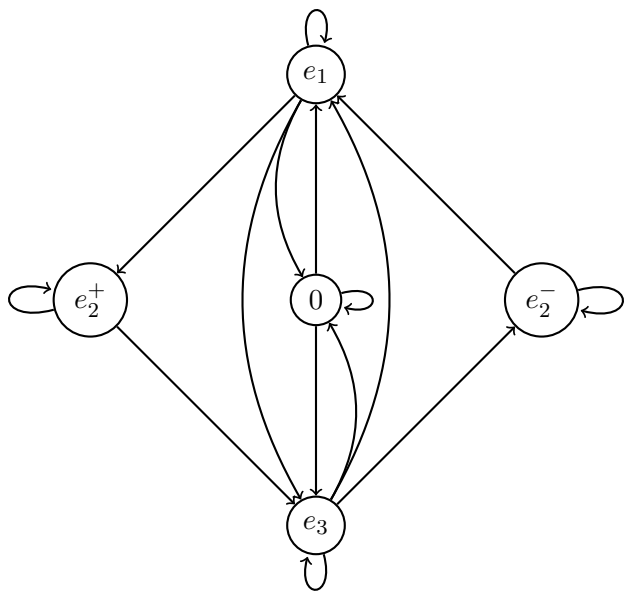
\begin{exercise}[Regular W-language generation]
  Draw a diagram, analogous to
  Figure~\ref{fig:topological-markov-chain} which describes only those
  sequences in which the rising and falling spans never stall, thus no
  frame repeats. We can call the resulting language a \emph{strict}
  W-language.  Assume that there are no breaks between symbols,
  i.e. connecting two consecutive 'X' or 'O' is mandatory.
\end{exercise}
\begin{exercise}[Higher resolution W-languages]
  Our W-language uses vertical resolution of 3 pixels. Define
  language $W_k$ in which symbols are $k$ pixels high. Consider
  the strict variant, also.
\end{exercise}
\section{A conventional W-language decoding algorithm}
Our next goal is to devise a simple algorithm which will correctly
decode the sequences encoded in the W-language. We emphasise that the
algorithm is ``conventional'' rather than ``connectionist'', although
the lines between these two approaches to programming will be
(deliberately) blurred in the following sections.

In order to correctly decode an image like in Figure~\ref{fig:binary-image}
processing it sequentially, by column, from left to right, we need to
detect and memorize the events associated with starting a new character.
The detection is possible by looking at a ``sliding window'' of 2 consecutive
column vectors.
\begin{center}
  \begin{tabular}{|c|c|p{2in}|}
    \hline
    First Column & Second Column & Event \\
    \hline
    $0$ or $e_1$ & $e_3$ &  Beginning of 'X'. \\
    $0$ or $e_3$ & $e_1$ &  Beginning of 'O'. \\
    \hline
  \end{tabular}
\end{center}
Let
\[ X=\begin{bmatrix}
    x_{1,1} & x_{1,2}\\
    x_{2,1} & x_{2,2}\\   
    x_{3,1} & x_{3,2}\\
  \end{bmatrix}
\]
be the sliding window. The beginning of 'X' is thus detected
by the logic statement:
\[ x_{2,1}=0 \wedge x_{3,1}=0 \wedge x_{3,2}=1 \]
Similarly, the beginning of 'O' is detected by  the logic statement:
\[ x_{2,1}=0 \wedge x_{1,1}=0 \wedge x_{1,2}=1 \]
These conditions can be expressed using auxillary variables:
\begin{align*}
  y_1&=\neg x_{2,1}\wedge \neg x_{3,1}\wedge x_{3,2}\\
  y_2&=\neg x_{2,1}\wedge \neg x_{1,1}\wedge x_{1,2}\\  
\end{align*}
The event can be recorded and memorized by setting variables $z_1$ and $z_2$
which indicate whether we are at the beginning of 'X' and 'O', respectively.
By convention, the meaning of the values of $z_1$ and $z_2$ is just given by:
\begin{center}
  \begin{tabular}{|c|c|}
    \hline
    Value of $z_1$ & Meaning\\
    \hline
    1 & Beginning of 'X'\\
    0 & Not beginning of 'X'\\    
    \hline
  \end{tabular}
  \begin{tabular}{|c|c|}
    \hline
    Value of $z_2$ & Meaning\\
    \hline
    1 & Beginning of 'O'\\
    0 & Not beginning of 'O'\\    
    \hline
  \end{tabular}
\end{center}
We also will use the vector $z=(z_1,z_2)$.
Knowing $z_1$ and $z_2$ allows us to emit 'X' or 'O' when we encounter the extreme values
$e_1$ and $e_3$. We observe that 'X' is emitted upon encountering a minimum in signal,
i.e. value $e_1$, while 'O' is emitted upon encountering a maximum, i.e. $e_3$.
The following table summarizes the actions which may result in emiting a symbol.
\begin{center}
  \begin{tabular}{|c|c|c|c|}
    \hline
    Conditions  & Value of $z_1$ & Value of $z_2$ & Action\\
    \hline
    $x_{1,2}=1$ & $0$ & $1$ & Emit 'X'\\
    $x_{3,2}=1$ & $1$ & $0$ & Emit 'O'\\
    All others & $*$ & $*$ & Nothing\\            
    \hline
  \end{tabular}
\end{center}
The action on every sliding window may result in setting the value of
$z_1$ or $z_2$ and/or emitting a symbol. Whether the symbol is emitted
or not will be signaled by setting a variable $emit_X$ or $emit_O$,
respectively.  In the algorithm, $emit_X$ and $emit_O$ are global
variables, their values persist outside the program. The program tells
the caller that an 'X' or 'O' was seen. The caller calls the program
on all frames (sliding windows) in succession, from left to right.

It is clear that an algorithm which correctly performs decoding should
look like Algorithm~\ref{alg:basic-decoding}. We divided the algorithm
into three sections, with horizontal lines.  These section nearly
exactly correspond to the three layers of the neural network
(deductron), which will be constructed from this program.
\begin{algorithm}[htb]
  \caption{A basic, handcrafted algorithm for decoding an image
    representing a sequence of 'X' and 'O'. \label{alg:basic-decoding}}
  \begin{algorithmic}[1]
    \Require
    \Statex The input parameter $X$ holds  the sliding window of the image with two consecutive columns.
    \Statex Global variables $z_1, z_2, emit_X, emit_O$. \Comment{Variables must persist outside this procedure.}
    \Ensure
    \Statex $emit_X$ is set to $1$ iff 'X' is detected in the input, else it is set to $0$;
    \Statex $emit_O$ is set to $1$ iff 'O' is detected in the input, else it is set to $0$.   
    \State $emit_X\gets 0$
    \State $emit_O\gets 0$
    \State $y_1=\neg x_{2,1}\wedge \neg x_{3,1}\wedge x_{3,2}$ \Comment{Set $y_1$ if start of 'X'.}
    \State $y_2=\neg x_{2,1}\wedge \neg x_{1,1}\wedge x_{1,2}$\Comment{Set $y_2$ if start of 'O'.} 
    \Divider
    \If{$y_1$}
    \State $z_1\gets 1$ \Comment{Remember we are in 'X'.}
    \State $z_2\gets 0$  \Comment{And remember we are not in 'O'.}  
    \ElsIf{$y_2$}
    \State $z_1\gets 0$ \Comment{Remember we are not in 'X'.}
    \State $z_2\gets 1$ \Comment{And remember we are in 'O'.}   
    \EndIf
    \Divider
    \State $emit_X\gets  x_{1,2}\wedge z_1$
    \State $emit_O \gets x_{3,2}\wedge z_2$    
  \end{algorithmic}
\end{algorithm}
Although we designed our algorithm to use a sliding window, this is
not necessary. (Hint: You can buffer your data from within your
algorithm, using persistent, i.e. global variables).

\begin{exercise}[Elimination of sliding window]
  Design an algorithm similar to
  Algorithm~\ref{alg:basic-decoding} which takes a single column
  (frame) of the image as input.
\end{exercise}  
\begin{exercise}[Pixel at a time]
  Design a similar algorithm to Algorithm~\ref{alg:basic-decoding}
  which takes a single pixel as input, assuming \emph{vertical
    progressive scan}: pixels are read from bottom-to-top, and then
  left-to-right.
\end{exercise}
\begin{exercise}[Counting algorithms]
  Count the number of distinct algorithms similar to
  Algorithm~\ref{alg:basic-decoding}. That is, count the
  algorithms which:
  \begin{enumerate}
  \item operate on a sliding window with $6$ pixels;
  \item use two 1-bit memory cells;
  \item produce two 1-bit outputs.
  \end{enumerate}
  Clearly, one of them is our algorithm.
\end{exercise}
\begin{exercise}[The precise topological Markov chain]
  Note that the transition graph in Figure~\ref{fig:topological-markov-chain}
  allows for generation of partial characters 'X' and 'O'. For example,
  the sequence:
  \[ 0, e_1, e_2^+, e_3, 0 \] would result in emitting an 'O' by our
  program, but the 'O' would never be completed. Prove that the
  transition diagram in Figure~\ref{fig:topological-markov-chain-mod}
  enforces completion of characters. In fact, prove that this
  topological Markov chain is 100\% compatible with
  Algorithm~\ref{alg:basic-decoding}. What is the role of superscripts ``f'' and ``s''?.
  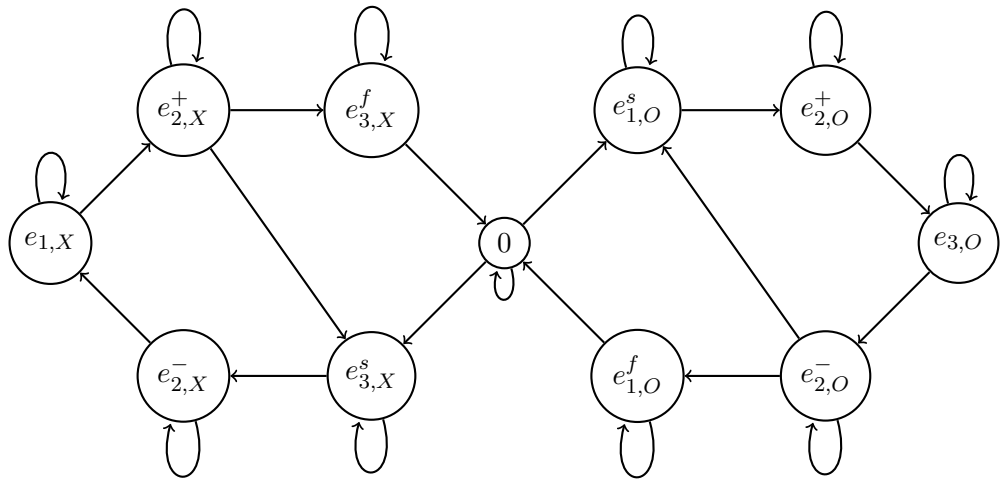
\begin{figure}[htb]
    \begin{center}
      \begin{tikzpicture}[scale=0.8, node distance=2.5cm,
        every edge/.style={thick,->,draw},
        every node/.style={circle,draw,thick}
        ]
        \node[circle](e1X){$e_{1,X}$};
        \node(e2pX)[circle, above right of = e1X] {$e_{2,X}^+$};
        \node(e2mX)[circle, below right of = e1X]{$e_{2,X}^-$};
        \node(e3X)[circle, right of = e2mX]{$e_{3,X}^s$};
        \node(e3XF)[circle, right of = e2pX] {$e_{3,X}^f$};
        \path (e1X) edge[->] (e2pX);
        \path (e2pX) edge[->] (e3X);
        \path (e3X) edge[->] (e2mX);
        \path (e2mX) edge[->] (e1X);            
        \path (e2pX) edge[->] (e3XF);

        \path (e1X) edge[loop above,->] (e1X);
        \path (e3X) edge[loop below,->] (e3X);
        \path (e3XF) edge[loop above,->] (e3XF);
        \path (e2pX) edge[loop above,->] (e2pX);
        \path (e2mX) edge[loop below,->] (e2mX);

        \node(zero)[circle, above right of = e3X] {$0$};
        \node(e1O)[circle, above right of = zero]{$e_{1,O}^s$};
        \node(e2pO)[circle, right of = e1O] {$e_{2,O}^+$};
        \node(e3O)[circle,  below right of = e2pO]{$e_{3,O}$};
        \node(e1OF)[circle,  below right of = zero]{$e_{1,O}^f$};
        \node(e2mO)[circle, right of = e1OF]{$e_{2,O}^-$};
        \path (e1O) edge[->] (e2pO);
        \path (e2pO) edge[->] (e3O);
        \path (e3O) edge[->] (e2mO);
        \path (e2mO) edge[->] (e1O);
        \path (e2mO) edge[->] (e1OF);

        \path (e1O) edge[loop above,->] (e1O);
        \path (e3O) edge[loop above,->] (e3O);
        \path (e2pO) edge[loop above,->] (e2pO);
        \path (e1OF) edge[loop below,->] (e1OF);
        \path (e2mO) edge[loop below,->] (e2mO);
        \path (zero) edge[loop below,->] (zero);
        \path(zero) edge[->] (e3X);
        \path(zero) edge[->] (e1O);
        \path(e3XF) edge[->] (zero);
        \path(e1OF) edge[->] (zero);

      \end{tikzpicture}
    \end{center}
    \caption{An improved topological Markov chain. The idea is to have
      essentially two copies of the diagram in
      Figure~\ref{fig:topological-markov-chain}, the left part for 'X'
      and the right part for 'O'. Also, some states are split to enforce
      completion of characters (the states with superscript ``f'' are  ``final''
      in generating each character).
      \label{fig:topological-markov-chain-mod}}
  \end{figure}

\end{exercise}

\begin{exercise}[Deductron and Chaotic Dynamics]
  In this exercise we develop what can be considered a custom
  pseudorandom number generator, which generates valid sequences in
  the W-language.  It mimics the operation of a \emph{linear
    congruential random number generator}
  (e.g. \cite{Knuth:1997:ACP:270146}, Chapter 3).  This exercise
  requires some familiarity with Dynamical Systems, for example, in
  the scope of Chapter 6 of \cite{alligood2000chaos}.
  Figure~\ref{fig:chaotic-mapping} we have an example of a simple
  chaotic dynamical system: a piecewise linear mapping of an interval
  $f:[0,\NumStates)\to[0,\NumStates)$. This mapping is
  \textbf{piecewise expanding}, i.e. $|f'(x)|>1$ except for the
  discontinuities.  In fact, $f'(x)\in\{2,4\}$. The intervals
  $[k,k+1)$, $k=0,1,\ldots,8$ are in 1:1 correspondence with the
  states of the Markov chain in
  Figure~\ref{fig:topological-markov-chain-mod}. This allows us to
  generate valid expressions of the W-language by using the dynamics
  of $f$.  We simply choose a random initial condition $x_0\in [0,9)$
  and create a \emph{trajectory} by successive applications of $f$:
  \[ x_{n+1} = f(x_{n}). \] Let $k_n$ be a sequence of numbers such
  that $x_n \in [k_n,k_n+1)$ for $n=0,1,\ldots$. Let $s_{n}$ be the
  corresponding sequence of states labeling the intervals, in the set
  \[\{
    e_{3,X}^s,e_{2,X}^-,e_{1,X},e_{2,X}^{+},e_{1,X},,e_{3,X}^f,
    ,\mathbf{0},
    e_{1,O}^s,e_{2,O}^+,e_{3,O},e_{2,X}^{-},e_{1,X}^f\}.\]
  The idea is the second subscript ('X' or 'O') keeps track of which
  symbol we are in the middle of. The superscript $\pm$ on
  $e_{2,X}^\pm$ and $e_{2,O}^\pm$ keeps track of whether we are rising
  or falling, as before. The superscripts 's' (for 'start') and 'f' (for 'finish')
  indicate whether we are starting or finishing the corresponding character.
  Thus, $e_{3,X}^s$ means we are starting an 'X', and $e_{3,X}^f$ means we are
  finishing an 'X'. The difference is that a finished character must be followed
  by a blank or the other character, and must not continue the same character.
  
  Prove that $s_0,s_1,\ldots,s_n,\ldots$ is a valid sentence the
  W-language. Conversely, show that for every such sentence there is
  an initial condition $x_0\in[0,\NumStates)$ reproducing this
  sentence. Moreover, for infinite sentences $x_0$ is unique.
  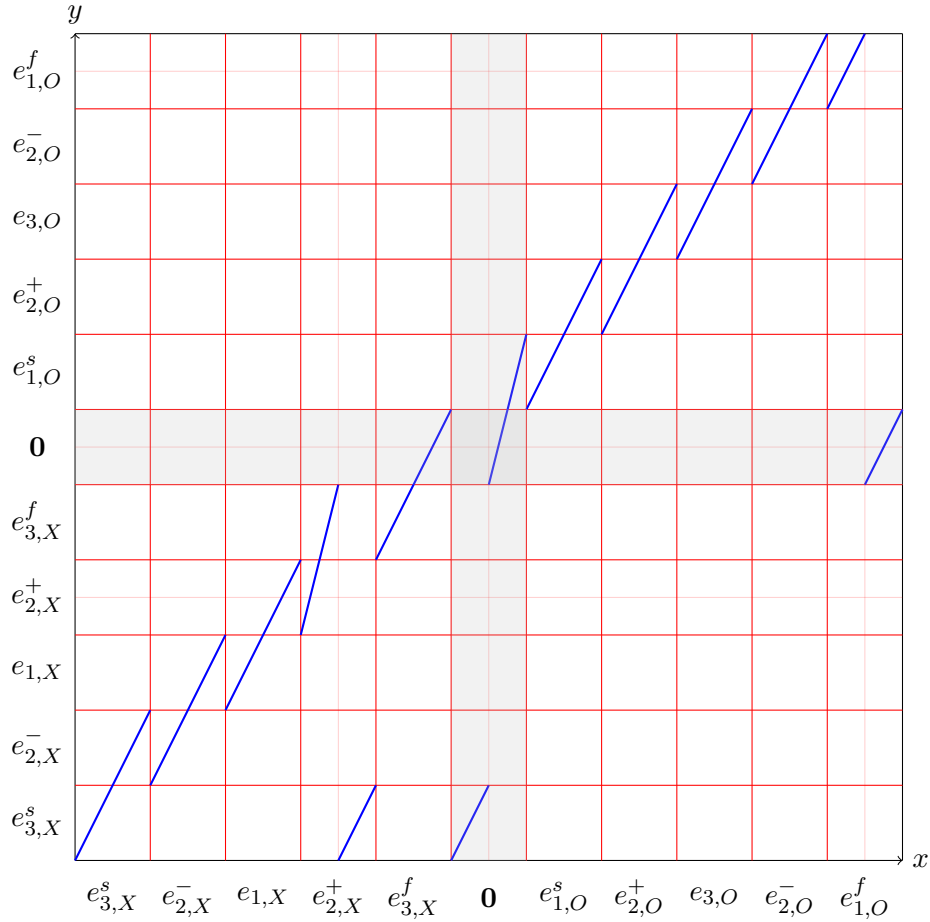
\begin{figure}
    \begin{tikzpicture}
      \pgfmathparse{\NumStates+1};
      \def\NumStatesPlusOne{\pgfmathresult};
      \draw[->] (0,0) -- (\NumStates,0) node[right] {$x$};
      \draw[->] (0,0) -- (0,\NumStates) node[above] {$y$};
      \draw[-] (\NumStates,0) -- (\NumStates,\NumStates)  node[right] {};
      \draw[-] (0,\NumStates) -- (\NumStates,\NumStates) node[above] {};
      \pgfmathparse{\NumStates-1};
      \def\NumStatesLessOne{\pgfmathresult};
      \foreach \x in {1,...,\NumStatesLessOne} {
        \draw[-,red] (\x,0) -- (\x,\NumStates);
        \draw[-,red] (0,\x) -- (\NumStates,\x);
      }
      \node at (0.5,-0.5){$e_{3,X}^s$};
      \node at (1.5,-0.5){$e_{2,X}^{-}$};
      \node at (2.5,-0.5){$e_{1,X}$};
      \node at (3.5,-0.5){$e_{2,X}^{+}$};
      \node at (4.5,-0.5){$e_{3,X}^f$};    
      \node at (-0.5,0.5){$e_{3,X}^s$};
      \node at (-0.5,1.5){$e_{2,X}^{-}$};
      \node at (-0.5,2.5){$e_{1,X}$};
      \node at (-0.5,3.5){$e_{2,X}^{+}$};
      \node at (-0.5,4.5){$e_{3,X}^f$};    
      \node at (6.5,-0.5){$e_{1,O}^s$};
      \node at (7.5,-0.5){$e_{2,O}^{+}$};
      \node at (8.5,-0.5){$e_{3,O}$};
      \node at (9.5,-0.5){$e_{2,O}^{-}$};
      \node at (10.5,-0.5){$e_{1,O}^f$};    
      \node at (-0.5,6.5){$e_{1,O}^s$};
      \node at (-0.5,7.5){$e_{2,O}^{+}$};
      \node at (-0.5,8.5){$e_{3,O}$};
      \node at (-0.5,9.5){$e_{2,O}^{-}$};
      \node at (-0.5,10.5){$e_{1,O}^f$};    
      \node at (5.5,-0.5) {$\mathbf{0}$};
      \node at (-0.5,5.5) {$\mathbf{0}$};

      \draw[domain=0:1,smooth,variable=\x,blue,thick] plot ({\x},{2*\x});
      \draw[domain=1:2,smooth,variable=\x,blue,thick] plot ({\x},{1+2*(\x-1)});
      \draw[domain=2:3,smooth,variable=\x,blue,thick] plot ({\x},{2+2*(\x-2)});
      \draw[domain=3:3.5,smooth,variable=\x,blue,thick] plot ({\x},{3+4*(\x-3)});
      \draw[-,red,very thin,opacity=0.2] (3.5,0)--(3.5, \NumStates);
      \draw[-,red,very thin,opacity=0.2] (0,3.5)--(\NumStates,3.5);
      \draw[domain=3.5:4,smooth,variable=\x,blue,thick] plot ({\x},{2*(\x-3.5)});
      \draw[domain=4:5,smooth,variable=\x,blue,thick] plot ({\x},{4+2*(\x-4)});
      \draw[domain=5:5.5,smooth,variable=\x,blue,thick] plot ({\x},{2*(\x-5)});
      \draw[-,red,very thin,opacity=0.2] (5.5,0)--(5.5, \NumStates);
      \draw[-,red,very thin,opacity=0.2] (0,5.5)--(\NumStates,5.5);
      \draw[domain=5.5:6,smooth,variable=\x,blue,thick] plot ({\x},{5+4*(\x-5.5)});

      \draw[domain=6:7,smooth,variable=\x,blue,thick] plot ({\x},{6+2*(\x-6)});
      \draw[domain=7:8,smooth,variable=\x,blue,thick] plot ({\x},{7+2*(\x-7)});
      \draw[domain=8:9,smooth,variable=\x,blue,thick] plot ({\x},{8+2*(\x-8)});
      \draw[domain=9:10,smooth,variable=\x,blue,thick] plot ({\x},{9+2*(\x-9)});
      \draw[domain=10:10.5,smooth,variable=\x,blue,thick] plot ({\x},{10+2*(\x-10)});
      \draw[domain=10.5:11,smooth,variable=\x,blue,thick] plot ({\x},{5+2*(\x-10.5)});
      \draw[-,red,very thin,opacity=0.2,opacity=0.2] (10.5,0)--(10.5, \NumStates);
      \draw[-,red,very thin,opacity=0.2] (0,10.5)--(\NumStates,10.5);

      \path[fill=gray!50,fill opacity=0.2] (5,0) -- (6,0) -- (6,\NumStates) -- (5,\NumStates) -- cycle;
      \path[fill=gray!50,fill opacity=0.2] (0,5) -- (0,6) -- (\NumStates,6) -- (\NumStates,5) -- cycle;
    \end{tikzpicture}
    \caption{A piecewise linear mapping $f:[0,\NumStates)\to[0,\NumStates)$, in which every piece has slope $2$ or $4$. The intervals $[k,k+1)$ are
      labeled with the states of the topological Markov chain in
      Figure~\ref{fig:topological-markov-chain-mod}.\label{fig:chaotic-mapping}}
  \end{figure}
\end{exercise}

\section{Converting a conventional program to a neural net}
Our ultimate goal is to construct a neural network which will decode
the class of valid inputs. A neural network does not evaluate logical
expressions and has no control structure of conventional
programs. Instead, it performs certain arithmetical calculations and
it outputs results based on hard or soft threshholding.

The next step towards a neural network consists in rewriting our
program so that it uses arithmetic instead of logic, and has no
control structures, such as ``if'' statements. We replace logical
variables with real variables, but initially \textbf{we restrict their
  values to $0$ and $1$ only}. It is important that the logical
operations (``and'', ``or'' and negation) are performed as arithmetic
on real values. 

The conditions in Algorithm~\ref{alg:basic-decoding} can be expressed
arithmetically (as every \emph{prepositional calculus formula}
can). We introduce the variables:
\begin{align*}
  y_1&=S(x_{2,1}+x_{3,1}+(1-x_{3,2}))\\
  y_2&=S(x_{2,1}+x_{1,1}+(1-x_{1,2}))\\  
\end{align*}
where $S$ is a function on integers defined by
\begin{equation}
  S(a) = \begin{cases}
    1 & a\leq 0,\\
    0 & a\geq 1.
  \end{cases}
\end{equation}
$S$ plays the role of an \emph{activation function}, in the language
of neural computing.  Variables $y_1$ and $y_2$ are conceptually
related to \emph{perceptrons}, or, in language closer to statistics,
they are \emph{binary linear classifiers}.  We note that this function
allows an easy test of whether a number of variables are
$0$. Variables $u_1,u_2,\ldots,u_r$ with values in the set $\{0,1\}$
are all zero iff
\[ S\left(\sum_{j=1}^r u_j\right) = 1 \]
We obtain Algorithm~\ref{alg:basic-decoding-arithmetical}.
\begin{algorithm}[htb]
  \caption{A version of the basic, handcrafted algorithm for decoding
    an image representing a sequence of 'X' and 'O' in which logical
    operations were replaced with
    arithmetic.\label{alg:basic-decoding-arithmetical}}
  \begin{algorithmic}[1]
    \Require
    \Statex The input parameter $X$ holds  the sliding window of the image with two consecutive columns.
    \Statex Global variables $z_1, z_2, emit_X, emit_O$. \Comment{Variables must persist outside this procedure.}
    \Ensure
    \Statex $emit_X$ is set to $1$ iff 'X' is detected in the input, else it is set to $0$;
    \Statex $emit_O$ is set to $1$ iff 'O' is detected in the input, else it is set to $0$.   
    \State $emit_X\gets 0$ \Comment{Initialize to $0$ (no emitting).}
    \State $emit_O\gets 0$ \Comment{Initialize to $0$ (no emitting).}
    \State $y_1=S(x_{2,1}+x_{3,1}+(1-x_{3,2}))$ \Comment{Set $y_1$ if start of 'X'.} 
    \State $y_2=S(x_{2,1}+x_{1,1}+(1-x_{1,2}))$ \Comment{Set $y_2$ if start of 'O'.} 
    \Divider
    \If{$y_1$}
    \State $z_1\gets 1$ \Comment{Remember we are in 'X'.}
    \State $z_2\gets 0$ \Comment{And remember we are Not in 'O'.}
    \ElsIf{$y_2$}
    \State $z_1\gets 0$ \Comment{Remember we are not in 'X'.}
    \State $z_2\gets 1$ \Comment{And remember we are in 'O'.}
    \EndIf
    \Divider
    \State $emit_X \gets S((1-x_{1,2})+(1-z_1))$
    \State $emit_O \gets S((1-x_{3,2})+(1-z_2))$    
  \end{algorithmic}
\end{algorithm}
The final adjustment to the algorithm is made in
Algorithm~\ref{alg:basic-decoding-arithmetical} in which \textbf{we replace
  all conditionals with arithmetic}.
This results in Algorithm~\ref{alg:basic-decoding-arithmetical-final}.
\begin{algorithm}[htb]
  \caption{A version of the basic, handcrafted algorithm for decoding
    an image representing a sequence of 'X' and 'O' in which all
    conditionals were converted to arithmetic. We can view this code
    as an algorithm calculating activations and outputs of a neural
    network with several types of
    neurons.\label{alg:basic-decoding-arithmetical-final} }
  \begin{algorithmic}[1]
    \Require
    \Statex The input parameter $X$ holds  the sliding window of the image with two consecutive columns.
    \Statex Global variables $z_1, z_2, emit_X, emit_O$. \Comment{Variables must persist outside this procedure.}
    \Ensure
    \Statex $emit_X$ is set to $1$ iff 'X' is detected in the input, else it is set to $0$;
    \Statex $emit_O$ is set to $1$ iff 'O' is detected in the input, else it is set to $0$.   
    \State $y_1 \gets S(x_{2,1}+x_{3,1}+(1-x_{3,2}))$ \Comment{Set $y_1$ if start of 'X'.} 
    \State $y_2 \gets S(x_{2,1}+x_{1,1}+(1-x_{1,2}))$ \Comment{Set $y_2$ if start of 'O'.} 
    \Divider
    \State $z_1 \gets (1-y_1)z_1 + y_1$ \Comment{If start of 'X', remember we are in 'X'.}
    \State $z_2 \gets (1-y_1)z_2$ \Comment{If start of 'X', remember we are not in 'O'.}
    \State $z_2 \gets (1-y_2)z_2 + y_2$ \Comment{If start of 'O', remember we are in 'O'.}
    \State $z_1 \gets (1-y_2)z_1$ \Comment{If start of 'O', remember we are not in 'X'.}
    \Divider
    \State $emit_X \gets S((1-x_{1,2})+(1-z_1))$ \Comment{Minimum and in 'X'.}
    \State $emit_O \gets S((1-x_{3,2})+(1-z_2)))$ \Comment{Maximum and in 'O'.}   
  \end{algorithmic}
\end{algorithm}
Upon close inspection, we can regard the algorithm as
an implementation of a neural network with several types of
neurons (gates).
\begin{enumerate}
\item Perceptron-type, with formula
  \[y_i=S\left(\sum_{j} w_{ij}x_j+b_i\right)\]
\item ``Forget and replace'' gate:
  \[ z'=U(y,z,z')\]
  where
  \[ U(y,z,z') =\begin{cases}
      z' & y=1,\\
      z  & y=0.
    \end{cases}
  \]
  Or arithmetically,
  \[
    U(y,z,z')=(1-y)\cdot z + y\cdot z'.
  \]
  This kind of gate provides a basic memory mechanism, where
  $z$ is preserved if $y=0$, or replaced with $z'$
  if $y=1$.
\end{enumerate}
Using the newly introduced U-gate we rewrite our main algorithm
as Algorithm~\ref{alg:basic-decoding-arithmetical-clean}.

\begin{algorithm}[htb]
  \caption{A version of the basic, handcrafted algorithm for decoding
    an image representing a sequence of 'X' and 'O'. Explicit gates
    are used to underscore the neural network
    format.\label{alg:basic-decoding-arithmetical-clean}}
  \begin{algorithmic}[1]
    \Require
    \Statex The input parameter $X$ holds  the sliding window of the image with two consecutive columns.
    \Statex Global variables $z_1, z_2, emit_X, emit_O$. \Comment{Variables must persist outside this procedure.}
    \Ensure
    \Statex $emit_X$ is set to $1$ iff 'X' is detected in the input, else it is set to $0$;
    \Statex $emit_O$ is set to $1$ iff 'O' is detected in the input, else it is set to $0$.   
    \State $y_1 \gets S(x_{2,1}+x_{3,1}+(1-x_{3,2}))$ \Comment{Set if start of 'X'.}
    \State $y_2 \gets S(x_{2,1}+x_{1,1}+(1-x_{1,2}))$ \Comment{Set if start of 'O'.}
    \Divider
    \State $z_1 \gets U(y_1,z_1,1)$ \Comment{If start of 'X', remember we are in 'X'.}
    \State $z_2 \gets U(y_1,z_2,0)$ \Comment{If start of 'X', remember we are not in 'O'.}    
    \State $z_2 \gets U(y_2,z_2,1)$ \Comment{If start of 'O', remember we are in 'O'.}
    \State $z_1 \gets U(y_2,z_1,0)$ \Comment{If start of 'O', remember we are not in 'X'.}    
    \Divider
    \State $emit_X \gets S((1-x_{1,2})+(1-z_1))$ \Comment{Minimum and in 'X'.}
    \State $emit_O \gets S((1-x_{3,2})+(1-z_2))$ \Comment{Maximum and in 'O'.}   
  \end{algorithmic}
\end{algorithm}

\section{An analysis of the U-gate and a new V-gate}
The U-gate implements in essence the \emph{modus ponens} inference
rule of prepositional logic:
\begin{equation*}
  p,\;{p \to q}\implies q
\end{equation*}
Indeed, $p$ represents the \emph{replace port} of a U-gate. The
assignment $q \gets U(p, q, 1)$ is equivalent to $p\to q$ in the
following sense:  the boolean variable $q$ represents a bit stored in
memory. If $p$ is true, $q$ is \emph{asserted}, i.e. set to true, so
that the logical expression $p\to q$ is true (has value $1$).
Similarly, the assigment $q \gets U(p,q,0)$ is equivalent to
$p\to \neg q$, i.e. $q$ is set to $0$, so that $p\to \neg q$ is
true. Thus, if $q$ is set to $1$, the fact $q$ is retracted, and the
fact $\neg q$ is asserted.  This semantics is similar to the semantics
of the Prolog system without variables, where we have a number of
\emph{facts}, such as ``$q$'' or ``$\neg q$'', in the Prolog database.
Upon execution, facts can be asserted or retracted from the database.

Thus, the inference layer consists of:
\begin{enumerate}
\item a number of variables $q_1, q_2,\ldots, q_n$ with some values of
  the variables set to either $0$ or $1$. Some of the variables may
  not be initialized, i.e. hold an undefined value;
\item a number of assignments $q_j \gets U(p_k, q_j, r_j)$ where
  $r_j=0$ or $r_j=1$, where the order of the assignments matters;
  the order may only be changed if the new order will always
  result in the same values for all variables after all assignments are
  processed; some assignments can be performed in parallel, if
  they operate on disjoint sets of variables $q_j$, so that
  the order of processing of the groups does not affect the result;
  the same variable $q_j$ may be updated many times by different U-gates.
  That is, later gates in the order may overwrite the result of the former
  gates.
\end{enumerate}
In the interaction between the variables $q_j$ and $r_j$, which are the
result of binary classification, and variables $p_j$, which represent
the memory of the system, it proves beneficial to assume that
$p_j$ is controlled by only two variables, and the final value
of $p_j$ after processing one input is represented by another kind
of gate, the $V$-gate, which combines the action of two $U$-gates.

The $V$-gate operates according to the formula:
\[ V(z, u, v) = (1-u)(1-v)z + u \]
where $z$ stands for a memory
variable (replacing $p$ in our naming convention used in the context
of the U-gate). A different (equivalent) formula for $V$-gate is:
\[
  V(z, u, v) = \begin{cases}
    z& \text{when $u=v=0$,}\\
    u& \text{when $v=1$.}
  \end{cases}
\]
Equivalently, in logic terms we have several wff's of propositional
calculus which represent $V$:
\begin{align*}
  V(z, u, v) &= u \vee (\neg v \wedge z) \\
             &= u \vee \neg ( v \vee \neg z)  \\
             & = (z \to v) \to u.
\end{align*}
The action of the variables $u$ and $v$ on $z$ is expressed
as the assignment:
\[ z \gets V(z, u, v). \]
We will adopt the following approach: every memorized variable $z$
will be controlled by exactly two variables: $u$ and $v$. The
rationale is that there are only two possible values of
$z$. Therefore, if multiple assignments are made to $z$, the final
result can be equivalently computed by combining those multiple
assignments. This is equivalent to performing conjunction of multiple
controlling variables $u$ and $v$. The conjunction can be done by
adding more variables to the first perceptron layer (adding together
activations is equivalent to the end operation). Hence, only one
$V$-gate is necessary to handle the change of the value of a memory
variable $z$.

The use of gate $V$ is illustrated by the following example:
\begin{example}[W-language decoding]
  In this example, we consider Algorithm~\ref{alg:basic-decoding-arithmetical-final}.
  Instead of using a $U$ gate, we can use the $V$-gate. Indeed, the assignments
  \begin{framed}
    \begin{algorithmic}
      \State $z_1 \gets (1-y_1)z_1 + y_1$ \Comment{If start of 'X', remember we are in 'X'.}
      \State $z_2 \gets (1-y_1)z_2$ \Comment{If start of 'X', remember we are not in 'O'.}
      \State $z_2 \gets (1-y_2)z_2 + y_2$ \Comment{If start of 'O', remember we are in 'O'.}
      \State $z_1 \gets (1-y_2)z_1$ \Comment{If start of 'O', remember we are not in 'X'.}
    \end{algorithmic}
  \end{framed}
  can be rewritten as:
  \begin{framed}
    \begin{algorithmic}
      \State $z_1 \gets (1-y_1)(1-y_2)z_1 + y_1$
      \State $z_2 \gets (1-y_1)(1-y_2)z_2 + y_2$
    \end{algorithmic}
  \end{framed}
  i.e.
  \begin{framed}
    \begin{algorithmic}
      \State $z_1 \gets V(z_1, y_1, y_2)$
      \State $z_2 \gets V(z_2, y_2, y_1)$
    \end{algorithmic}
  \end{framed}
  Hence, $y_1$ and $y_2$ are controlling both $z_1$ and $z_2$.
\end{example}
Algorithm~\ref{alg:three-layer-final} is a modification of the
previous algorithms which does not use the input values in the output
layer.  This is achieved by using the input layer (binary
classification of the inputs) to memorize some input values in the
memories (variables $z_j$). This technique demonstrates that the
output layer of a deductron performing only binary classification of
the memories (variables $z_j$) is sufficiently general without
explicitly utilizing input values.
\begin{algorithm}
  \caption{The final three-layer deductron architecture utilizing 4
    memory cells.\label{alg:three-layer-final}}
  \begin{algorithmic}[1]
    \Require
    \Statex The input parameter $X$ holds  the sliding window of the image with two consecutive columns.
    \Statex Global variables $z_1, z_2, z_3, z_4, emit_X, emit_O$. \Comment{Variables must persist outside this procedure.}
    \Ensure
    \Statex $emit_X$ is set to $1$ iff 'X' is detected in the input, else it is set to $0$;
    \Statex $emit_O$ is set to $1$ iff 'O' is detected in the input, else it is set to $0$.   

    \State $y_{1,1} \gets S(x_{2,1}+x_{3,1}+(1-x_{3,2}))$ \Comment{$=y_1$.}
    \State $y_{2,1} \gets S(x_{1,1}+x_{2,1}+(1-x_{1,2}))$ \Comment{$=y_2$.}
    \State $y_{3,1} \gets S(x_{1,1}+(1-x_{1,2}))$ \Comment{Set if $x_{1,2}\wedge\neg x_{1,1}$.}
    \State $y_{4,1} \gets S(x_{3,1}+(1-x_{3,2}))$ \Comment{Set if $x_{3,2}\wedge\neg x_{3,1}$.}

    \State $y_{1,2} \gets S(x_{1,1}+x_{2,1}+(1-x_{1,2}))$ \Comment{$=y_2$.}
    \State $y_{2,2} \gets S(x_{2,1}+x_{3,1}+(1-x_{3,2}))$ \Comment{$=y_1$.}
    \State $y_{3,2} \gets S(0)$ \Comment{$=1$}
    \State $y_{4,2} \gets S(0)$ \Comment{$=1$}           

    {\color{red}\hrule}
    \State $z_1 \gets V(z_1,y_{1,1},y_{1,2})$ 
    \State $z_2 \gets V(z_2,y_{2,1},y_{2,2})$ 
    \State $z_3 \gets V(z_3,y_{3,1},y_{3,2})$ \Comment{$z_3\gets y_{3,1}$.}
    \State $z_4 \gets V(z_4,y_{4,1},y_{4,2})$ \Comment{$z_4\gets y_{4,1}$.}       

    {\color{red}\hrule}
    \State $emit_X \gets S((1-z_3)+(1-z_1))$ 
    \State $emit_O \gets S((1-z_4)+(1-z_2))$ 
  \end{algorithmic}
\end{algorithm}

Let us finish this section with a mathematical result proven by our
approach:
\begin{theorem}[On deductron decoding of W-language]
  There exists a deductron with $4$ memory cells which correctly decodes
  every valid expression of the W-language.
\end{theorem}
\begin{proof}
  As we constructed the deductron by writing an equivalent pseudocode,
  we prove first that one of the presentations of the algorithm, e.g.,
  Algorithm~\ref{alg:basic-decoding}, decodes the W-language
  correctly. The proof is not difficult and it uses the formal
  definition, which is essentially
  Figure~\ref{fig:topological-markov-chain}. The tools to do so, such
  as invariants, are standard in computer science.  Another part of
  the proof is to show that the neural network yields the same
  decoding as the pseudocode, even if the real arithmetic is only
  approximate. The details are left to the reader.
\end{proof}
\begin{exercise}[$3$ memory cells suffice for W-language]
  Prove that there exists a deductron with $3$ memory cells
  correctly decoding W-language. For instance, write a different
  conventional program which uses fewer variables, and convert
  it to a $3$-cell deductron.
\end{exercise}
\begin{exercise}[$2$ memory cells insufficient for W-language]
  Prove that there is no deductron with $2$ memory cells, which correctly
  decodes every expression of the $W$-language.
\end{exercise}
\begin{exercise}[$2$ memory cells suffice for strict W-language]
  Prove that there is a deductron with $2$ cells, which correctly
  decodes every expression of the strict $W$-language.
\end{exercise}
\begin{exercise}[$1$ memory cell insufficient for strict W-language]
  Prove that there is no deductron with $1$ memory cell, which correctly
  decodes every expression of the strict $W$-language.
\end{exercise}
As a hint for the previous exercises, we suggest studying Shannon
information theory. In particular, the Channel Coding Theorem gives us
the necessary tools to obtain a bound on the number of memory
cells. Essentially, the memory is the ``bottleneck'' for passing
information between inputs and outputs. Of course, information is
measured in bits and it does not need to be a whole number.

Upon considering the structure of the neural network based on
perceptron layers and the new V-gate seen in
Figure~\ref{fig:deductron}, we can see that our network is a 3-layer
network. The first and third layer are perceptron layers, thus
performing binary linear classification. We will call the first
perceptron layer \emph{the input layer} and the third layer \emph{the
  output layer}.

\begin{figure}
  \begin{center}
    \usetikzlibrary{shapes.geometric}
\usetikzlibrary{arrows,positioning,shapes.geometric,calc,decorations}

\tikzstyle{demux} = [ trapezium, draw,   
shape border rotate = 90, trapezium angle = 60,  
inner ysep=0pt, outer sep=0pt, inner xsep=2pt, 
text width = 3em, 
node distance=2.5cm, font=\large, align=right ]

\tikzstyle{terminal} = [rectangle, draw, rounded corners=5pt, node distance=2.5cm ]

\tikzstyle{v-gate} = [rectangle, draw, text width = 3em, node distance=3cm, font = \large,
minimum height=2cm, align=center]

\tikzstyle{memory} = [node distance=3cm]

\tikzstyle{binary-classifier} = [rectangle, draw, text width = 2em, node distance = 2.5cm, font = \large,
minimum height=2cm, align=center]

\tikzstyle{weight-bias} = [rectangle, rounded corners = 10pt, draw, node distance = 2.5cm, font = \large,
minimum size=2cm, align=center]

\tikzset{
  pics/.cd,
  deductron/.style = {
    code={
      \node[terminal] (input-#1) {$in_{#1}$};
      \node[binary-classifier, right of=input-#1] (input-classifier-#1) {$P_{#1}'$};
      \node[demux, right of = input-classifier-#1] (demux-#1) {L\smallskip \\ \smallskip R};
      \coordinate (left-#1) at ($(demux-#1.east) + (0, 0.5)$);
      \coordinate (right-#1) at ($(demux-#1.east) + (0, -0.5)$);
      \node[v-gate,right of = demux-#1](v-gate-#1) {$V$};
      \coordinate (memory-#1) at ($(v-gate-#1.east) + (2,0)$);
      \node[binary-classifier,right of = memory-#1] (output-classifier-#1) {$P_{#1}''$};
      \node[terminal, right of=output-classifier-#1] (output-#1) {$out_{#1}$};

      \path[draw,->] (input-#1) --  (input-classifier-#1.west)  node[below, near end]{$x_{#1}$};
      \path[draw,->] (input-classifier-#1) -- (demux-#1) node[below, near end]{$h_{#1}$};
      \path[draw,->] (left-#1) -- ($(v-gate-#1.west)+(0,0.5)$) node[below, near end]{$u_{#1}$};
      \path[draw,->] (right-#1) -- ($(v-gate-#1.west)+(0,-0.5)$) node[below, near end]{$v_{#1}$};
      \path[draw,->] (v-gate-#1) -- (memory-#1) node[below, near start]{$z_{#1}$};
      \path[draw,->] (memory-#1) -- (output-classifier-#1) node[below, near end]{$z_{#1}$};
      \path[draw,->] (output-classifier-#1) -- (output-#1) node[below, near end]{$o_{#1}$};
    }
  }
}

\begin{tikzpicture}
  \foreach \y in {0, 1, ..., 4}
  \pic[scale=.8] at (0,-3*\y) {deductron=\y};
  \node[rectangle,draw,node distance=3cm, above of=memory-0](memory--1){$0$};

 \foreach \y [count=\yy from -1] in {0,1,2,...,4}
 \path[draw,->] (memory-\yy) -- ++ (0,-1.3) -| (v-gate-\y) node[left, near end]{$z_{\yy}$};

 \node[weight-bias,above of=input-classifier-0](input-weight-bias){$(W_1,b_1)$};
 \node[weight-bias,above of=output-classifier-0](output-weight-bias){$(W_2,b_2)$}; 

 \path[draw,red!90,thick] (input-weight-bias)-- ++(-1.5,0) |- ++(0,-14);
 \path[draw,red!90,thick] (output-weight-bias)-- ++(-1.5,0) |- ++(0,-14); 

 \foreach \y in {0,1,2,...,4} {
   \path[draw,->,red!90,thick] ($(input-classifier-\y.west) +(-1,0.5)$) -- ($(input-classifier-\y.west)+(0,0.5)$);
   \path[draw,->,red!90,thick] ($(output-classifier-\y.west) +(-1,0.5)$) -- ($(output-classifier-\y.west)+(0,0.5)$);
 }

\end{tikzpicture}

  \end{center}
  \caption{The deductron neural network architecture suitable for
    BPTT.  The deductron is replicated as many times as there are
    inputs (frames).  The perceptron $P_j'$ classifies the input
    vectors $x_j$, which are subsequently demuxed and fed into the
    V-gates $V_j$, along with the content of the memory $z_{j-1}$, and
    the output is sent to memory $z_j$.  The output of $z_j$ is fed
    into the output classifier $P_j''$.  and output vector $o_j$ is
    produced. Perceptrons $P_j'$ share common weights $W_1$ and biases
    $b_1$. Perceptrons $P_j''$ share common weights $W_2$ and biases
    $b_2$. \label{fig:deductron}}
\end{figure}
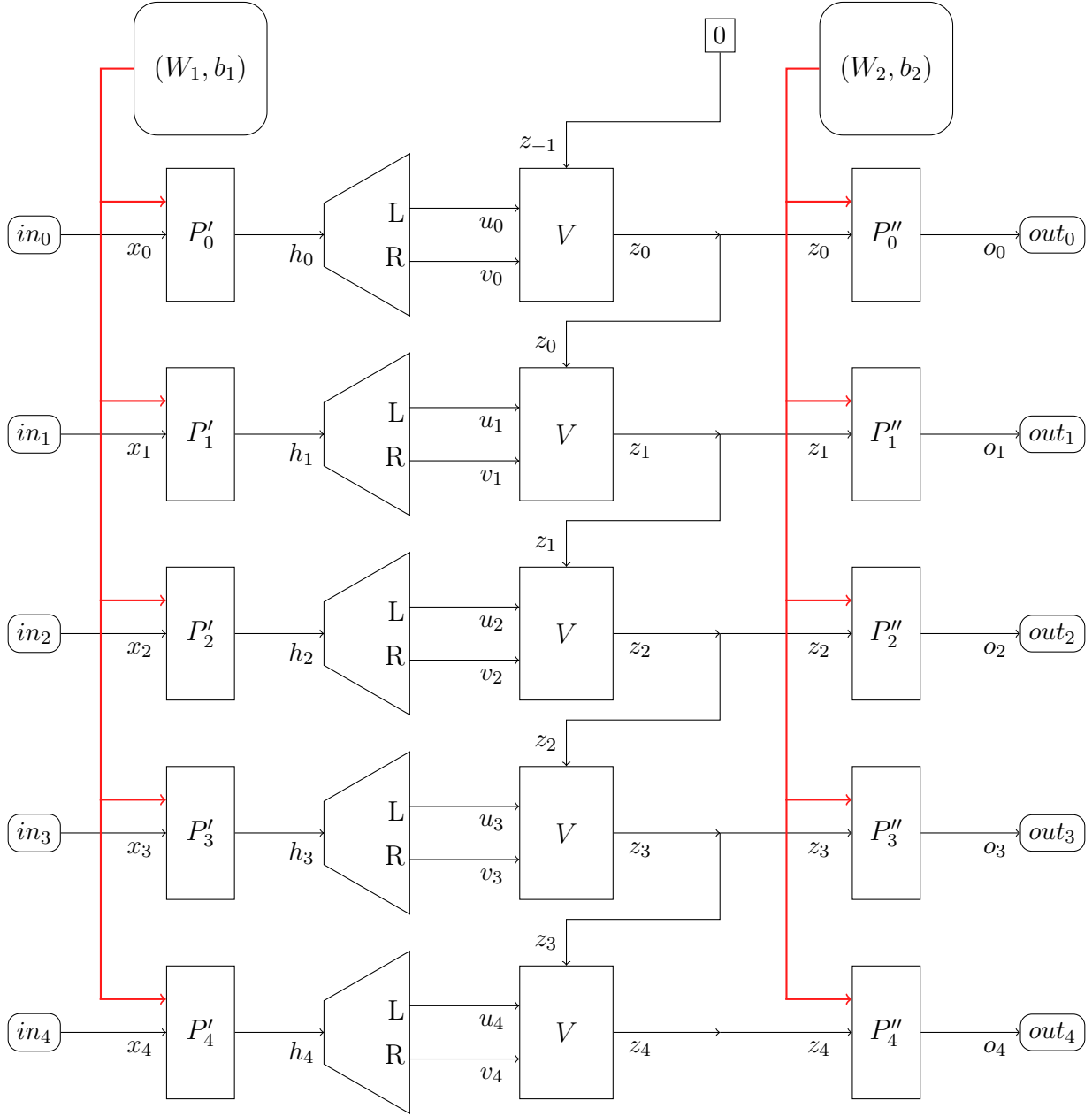

The middle layer is a new layer containing V-gates. We will call this
layer the \emph{inference layer}, as indeed it is capable of formal
deduction of predicate calculus. We now proceed to justify this
statement.

The general architecture based on $V$ gate is quite simple and it comprises:
\begin{enumerate}
\item 
  the input perceptron layer, producing $n_{hidden}=2\,n_{memory}$ paired values
  $u_i$ and $v_i$, $i=1,2,\ldots, n_{memory}$, where $n_{memory}$ is
  the number of memory cells;
\item
  the inference (memory) layer, consisting of $n_{memory}$ memory cells whose
  values persist until modified by the action of the V-gates; the update rule
  for the memory cells is
  \[ z_i \gets V(z_i, u_i, v_i) \]
  for $i=1,2,\ldots, n_{memory}$.
\item
  the output perceptron layer, which is a binary classifier
  working on the memory cells.
\end{enumerate}
\begin{algorithm}
  \caption{A $V$-gate simulator. We note that
    $V(z, u, v)=(1-u)(1-v)z + u$. If $z$, $u$ and $v$ are vectors
    of equal length, all operations are performed elementwise.\label{alg:v-gate} }
  \begin{algorithmic}[1]
    \Function{VGate}{$z$, $u$, $v$}
    \State\Return $(1-u)(1-v)z+u$
    \EndFunction
  \end{algorithmic}
\end{algorithm}

The semantics of the neural network can be described by the simulation
algorithm, Algorithm~\ref{alg:architecture-simulator} which expresses
the process of creation of outputs as standard pseudocode.

Figure~\ref{fig:deductron} is a rudimentary systems diagram and can be
considered a different presentation of
Algorithm~\ref{alg:architecture-simulator}, focused on movement of
data in the algorithm. This is especially useful to building circuits
for training deductrons.

\begin{algorithm}
  \caption{A simulator for our $3$-layer deductron network. Note: The
    activation function $S$ operates on vectors
    elementwise. \label{alg:architecture-simulator}}
  \begin{algorithmic}[1]
    \Require
    \Statex $W_1$ is a $2n_{memory}\times n_{in}$ matrix;
    \Statex $b_1$ is a $2n_{memory}\times 1$ column vector;
    \Statex $W_2$ is a $n_{out}\times n_{memory}$ matrix;  
    \Statex $b_2$ is a $n_{out}\times 1$ column vector;
    \Statex $x = [x_0,x_1,\ldots,x_{n_{frames}-1}]$ is a list of $n_{frames}$ frames, which
    are $n_{in}\times 1$ column vectors.
    \Ensure
    \Statex $o=[o_0,o_1,\ldots,o_{n_{frames}-1}]$ is set to a list of $n_{frames}$ outputs, which
    are $n_{out}\times 1$ vectors.

    \For{$t=0,1,2,\ldots,n_{frames}-1$}
    \State $h_{t} = S\left(W_1 x_{t} + b_1\right)$ \Comment{Classify inputs.}
    \State $(u_t, v_t) \gets \Call{Split}{h_t}$ \Comment{Splits (demuxes) $h_{t}$ into 2 vectors of equal length.}
    \EndFor
    \Divider
    \State $z_0 \gets 0$        \Comment{Initialize memory to $0$; $z_0$ is an $n_{memory}\times 1$ column vector.}
    \For{$t=1,2,\ldots,n_{frames}-1$}
    \State $z_{t} \gets \Call{VGate}{z_{t-1}, u_{t}, v_{t}}$ \Comment{$z_t$ is an $n_{memory}\times 1$ column vector.}
    \EndFor
    \Divider
    \For{$t=0,1,2,\ldots,n_{frames}-1$}
    \State $o_{t} = S\left(W_2z_{t} + b_2\right)$ \Comment{Classify memories, produce outputs.}
    \EndFor
  \end{algorithmic}
\end{algorithm}

\section{Interpretability of the weights as logic formulas}
Obviously, it would be desirable if the weights found by a computer
could be interpreted by a human as ``reasonable steps'' to perform the
task. In most cases, formulas obtained by training a neural network
cannot be interpreted in this manner. For once, the quantity of
information reflected in the weights may be too large for such an
interpretation. Below we express some thoughts particular to training
the deductron using simulated annealing on the W-language.

\begin{table}
  \caption{Optimal weights learned by simulated annealing.  We varied
    the inverse temperature from $0$ to $10$, with step $10^{-5}$,
    resulting in $\approx 1$ million iterations.  The number of memory
    units is $3$, and the training input was the one in
    Fig~\ref{fig:binary-image}. \label{fig:learned-weights}}
  \begin{small}
\begin{verbatim}
>>> test_annealing(small_inputs, small_targets)
**** Simulated annealing ****
 Iteration       Loss  Best Loss Inv. Temp.
       000     14.500     14.500      0.000
       001     14.500     14.500      0.000
       015     14.497     14.500      0.000
[ Many lines of output skipped...]
    999000      0.002      0.001      9.990
    999179      0.001      0.001      9.992
    999180      0.001      0.001      9.992
    999187      0.001      0.001      9.992
    999189      0.001      0.001      9.992
    999192      0.001      0.001      9.992
    999193      0.001      0.001      9.992
    999197      0.001      0.001      9.992
    999323      0.001      0.001      9.993
DeductronBase:
beta: 9.993229999790847
shift: 0.5
W1:
[[ 0  1  0 -1  1  0]
 [-1  0 -1  1  0  0]
 [ 1  1  0  0 -1  1]
 [-1  1  0  0  1 -1]
 [ 1  0  0 -1 -1  0]
 [ 1  0 -1  1 -1 -1]]
B1:
[[1]
 [0]
 [0]
 [2]
 [0]
 [1]]
W2:
[[ 1  1 -1]
 [-1 -1  1]]
B2:
[[1]
 [2]]
\end{verbatim}
  \end{small}
\end{table}

In Table~\ref{fig:learned-weights} we see the weights found by
simulated annealing. When a bias equals the number of $-1$'s in the
corresponding row of the matrix, it is apparent that that row of
weights corresponds to a formula of logic (conjunction of inputs or
their negations). However, in some runs (due to randomization), we
obtain weights which do not correspond to logic formulas.  Clearly,
some of the rows of $W1$ and $B1$ do not have this property.

\begin{example}[Weights and biases obfuscating a simple logic formula]
  Let us consider weights and biases obtained in one numerical experiment:
  \begin{enumerate}
  \item a row of weights $[1, 1, 0, -1, 0, 1]$;
  \item bias $3$.
  \end{enumerate}
  Since two of the weights are $1$, with a single weight
  of $-1$, the activation computed using it is at least $2$.  The
  activation is in the region where $S$ yields a near-zero. Hence, the
  hidden unit constantly yields $0$ (false), thus is equivalent to a
  simple, trivial propositional logic formula ($\False$).
\end{example}
\begin{example}[Weights and biases without an equivalent conjunction]
  Let us consider weights and biases obtained in one numerical experiment:
  \begin{enumerate}
  \item a row of weights $[0, 1, -1, 1, ,1 -1]$;
  \item bias $1$.
  \end{enumerate}
  Thus the activation is $x_1-x_2+x_3+x_4-x_5 + 1$. Assuming that $x_j\in\{0,1\}$,
  there is no conjunction of $x_j$, $\neg x_j$, or $\True$, $j=0,1,\ldots,5$,
  equivalent to this arithmetic formula (the reader is welcome to prove this).

  Nevertheless, there is a complex logical formula which is a
  disjunction of conjunctions, true only for solutions of this
  equation. This demonstrates that the logical formulas expressing the
  arithmetic equation can be more complex than just conjunctions, as in
  our manually constructed program. It is clear that any arithmetic
  linear equation or inequality over rational numbers can be expressed as a
  single logical formula in disjunctive normal form (disjunction of
  conjunctions).
\end{example}
\begin{exercise}[Disjunction of conjunctions for a linear inequality]
  Consider the linear inequality
  \[ x_1-x_2+x_3+x_4-x_5 + 1 > 0 \]
  over the domain $x_j\in\{0,1\}$, $j=0,1,\ldots,5$. Construct
  an equivalent logical formula, which is a conjunction of disjunctions
  of some of the statements $x_j=0$ or $x_j=1$.
\end{exercise}

Algorithm~\ref{alg:architecture-simulator} defines a class of programs
parameterized by weights and biases. The program expressed by
Algorithm~\ref{alg:basic-decoding-arithmetical-final} can be obtained
by choosing the entries of the weight matrices
$W_1=\left[w_{ij}^{(1)}\right]$ and $W_2 =\left[w_{ij}^{(2)}\right]$
and the bias vectors $b_1=\left[b_i^{(1)}\right]$ and
$b_2=\left[b_i^{(2)}\right]$ so that:
\begin{enumerate}
\item each weight $w_{ij}^{(k)}$, $k=1,2$, is chosen to be $\pm 1$ or $0$;
\item each entry $b_{j}^{(k)}$ is chosen to be the count of $-1$'s in
  the $i$-th row of the matrix $W_k$.
\end{enumerate}
With these choices, the matrix product expresses the value of a formula
of propositional calculus. This is implied by the following:
\begin{lemma}[Arithmetic vs. logic]
  \label{thm:arithmetic-vs-logic}
  Let $W$, $b$, $x$ and $y$ be real matrices such that:
  \begin{enumerate}
  \item $W=\left[w_{ij}\right]$, $1\leq i \leq m$, $1\leq j \leq n$, $w_{ij}\in\{-1,0,1\}$; 
  \item $x=\left[x_j\right]$, $1\leq j\leq n$, $x_j\in\{0,1\}$;
  \item $b=\left[b_i\right]$, $1\leq i\leq m$, where $b_i$ is the count of $-1$ amongst $w_{i1},w_{i2},\ldots,w_{in}$;
  \item $y = Wx$.
  \end{enumerate}
  Then
  \[ y_i = \sum_{j=1}^{n} w_{ij} x_j = \bigwedge_{j=1}^{n} x_{ij}, \qquad 1\leq i \leq m \]
  where
  \begin{equation*}
    x_{ij} =
    \begin{cases}
      x_j      & w_{ij} = 1,\\
      \neg x_j & w_{ij} = -1,\\
      \True    & w_{ij} = 0.
    \end{cases}
  \end{equation*}
\end{lemma}
\begin{proof}
  Left to the reader.
\end{proof}
\begin{exercise}[A formula for biases]
  Prove that under the assumptions of Lemma~\ref{thm:arithmetic-vs-logic}.
  \[ b_i = \sum_{j=1}^n g(w_{ij}) \]
  where
  \[ g(w) = \frac{w(w-1)}{2} = \binom{w}{2}.\]
\end{exercise}

\section{Machine learning}
It remains to demonstrate that the neural network architecture is
useful, i.e., that it represents a useful class of programs, and that
the programs can be learned automatically. To demonstrate supervised
learning, we applied simulated annealing to learn the weights of a
program which will solve the decoding problem for the W-language with
100\% accuracy.

We used the target vector $t$ corresponding to the input presented in
Figure~\ref{fig:sample2}. The target vector is simply the output
of the handcrafted decoding algorithm.

We restricted the weights to values $\pm1$. The biases were restricted
to the set $\{0,1,2,3,4,5\}$. The \emph{loss (error) function} is the quantity
\[ loss = \sum_{f=1}^{n_{frames}}\sum_{i=1}^{n_{out}} \left|t_i^{(f)}-o_i^{(f)}\right|^{\gamma} \]
(we only used $\gamma=1$ and $\gamma=2$ in the current paper, with approximately the same results)
where $n_{frames}$ represent the number of $6$-pixel frames constructed
by considering a sliding window of $2$ consecutive columns of the image.
We note that $t^{(f)}$ and $o^{(f)}$ are the target and output
vectors for frame $f$, respectively. It should be noted
that sequences are fed to the deductron in a specific order,
in which the memory will be updated, thus the order cannot be changed.
For zero temperatures, $t$ and $o$ are vectors with values $0$ and $1$
and the energy function reduces to the Hamming distance.

The energy function is thus the function of the weights. The perceptron
activation function was set to
\begin{equation}
  \label{eqn:falling-sigmoid}
  S(x) = \frac{1}{1 + \exp(\beta(x-0.5))}
\end{equation}
with a graph portrayed in Figure~\ref{fig:falling-sigmoid}:
\begin{figure}[htb]
  \caption{The sigmoid function $S(x)=\frac{1}{1+\exp(\beta(x-0.5))}$
    for $\beta=5$ and $\beta=15$.\label{fig:falling-sigmoid}}
  \begin{center}
    \begin{tikzpicture}
      \begin{axis}[axis lines=left, xmin=0,xmax=1,ymin=-0,ymax=1,samples=300]
        \addplot[blue, ultra thick]  { 1 / ( 1 + exp ( 5 * (x - 0.5) ) )} ;
        \addplot[green, ultra thick]  { 1 / ( 1 + exp ( 15 * (x - 0.5) ) )} ;        
      \end{axis}
    \end{tikzpicture}
  \end{center}
\end{figure}
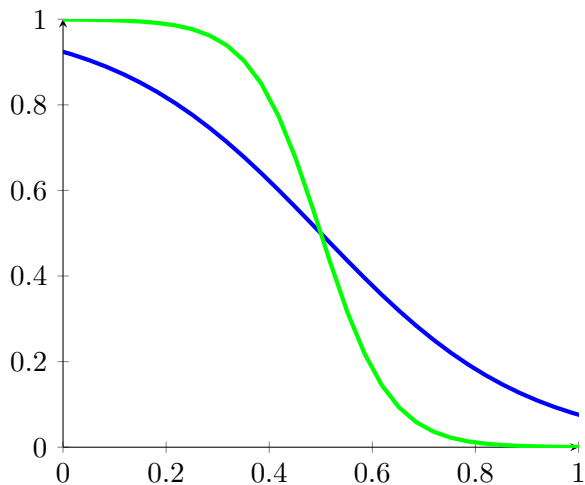

Here $\beta$ represents the \emph{inverse temperature} of simulated
annealing. This sigmoid function in the limit $\beta\to\infty$ becomes
the function
\[ S(x) = \begin{cases}
    1 & \text{when $x<1/2$,}\\
    0 & \text{when $x>1/2$,}\\
    1/2& \text{when $x=1/2$.}
  \end{cases}
\]
The simulated annealing program finds the system of weights presented
in Figure~\ref{fig:learned-weights}.

\begin{table}
  \caption{Optimal weights found by reading them off from the the
    handcrafted decoder
    (Algorithm~\ref{alg:basic-decoding-arithmetical-final}).
    \label{fig:hand-coded-weights}}
  \begin{small}
\begin{verbatim}
WLangDecoderExact:
beta: 10
shift: 0.5
W1:
[[ 0.  1.  1.  0.  0. -1.]
 [ 1.  1.  0. -1.  0.  0.]
 [ 1.  0.  0. -1.  0.  0.]
 [ 0.  0.  1.  0.  0. -1.]
 [ 1.  1.  0. -1.  0.  0.]
 [ 0.  1.  1.  0.  0. -1.]
 [-1.  0.  0.  0.  0.  0.]
 [ 0.  0. -1.  0.  0.  0.]]
B1:
[[1.]
 [1.]
 [1.]
 [1.]
 [1.]
 [1.]
 [1.]
 [1.]]
W2:
[[-1.  0. -1.  0.]
 [ 0. -1.  0. -1.]]
B2:
[[2.]
 [2.]]
\end{verbatim}
  \end{small}
\end{table}

As it is seen, the energy was reduced to approximately $10^{-3}$ which
is a guarantee that all responses have been correct. The outputs are
presented alongside with inputs in
Figure~\ref{fig:learned-weights-output}. For comparison, the weights
directly read from the program
Algorithm~\ref{alg:basic-decoding-arithmetical-final} are in
Figure~\ref{fig:hand-coded-weights}.  Clearly, the weights learned by
simulated annealing differ from the handcrafted weights.  However,
they both reproduce equivalent results.  Interestingly, both programs
correctly decode output of the topological Markov chain presented in
Figure~\ref{fig:topological-markov-chain}, with approximately 500
frames. The sample constructed contains ``stretched'' characters 'X'
and 'O' obtained by repeating falling and rising spans of random
length. Thus, the weights constructed by simulated annealing learned
how to solve the more general problem than indicated by the sole
example used as a training set.
\begin{table}
  \caption{\label{fig:learned-weights-output}The output of the simulator using
    optimal weights constructed by simulated annealing (Figure~\ref{fig:learned-weights}).
    The first six columns contain the inputs (linearized sliding windows)
    for input depicted in Figure~\ref{fig:binary-image}.
    The next two columns are the target values. Finally, we
    identfy steps where we emit 'X' or 'O' in the right column.}
  \begin{small}
\begin{verbatim}
  0:    0    0    0    0    0    1 	   0    0 	
  1:    0    0    1    0    1    0 	   0    0 	
  2:    0    1    0    1    0    0 	   1    0 	emit: X
  3:    1    0    0    0    1    0 	   0    0 	
  4:    0    1    0    0    0    1 	   0    0 	
  5:    0    0    1    1    0    0 	   0    0 	
  6:    1    0    0    0    1    0 	   0    0 	
  7:    0    1    0    0    0    1 	   0    1 	emit: O
  8:    0    0    1    0    1    0 	   0    0 	
  9:    0    1    0    1    0    0 	   0    0 	
 10:    1    0    0    0    1    0 	   0    0 	
 11:    0    1    0    0    0    1 	   0    1 	emit: O
 12:    0    0    1    0    1    0 	   0    0 	
 13:    0    1    0    1    0    0 	   0    0 	
 14:    1    0    0    0    0    1 	   0    0 	
 15:    0    0    1    0    1    0 	   0    0 	
 16:    0    1    0    1    0    0 	   1    0 	emit: X
 17:    1    0    0    0    1    0 	   0    0 	
 18:    0    1    0    0    0    1 	   0    0 	
 19:    0    0    1    0    1    0 	   0    0 	
 20:    0    1    0    1    0    0 	   1    0 	emit: X
 21:    1    0    0    0    1    0 	   0    0 	
 22:    0    1    0    0    0    1 	   0    0 	
 23:    0    0    1    1    0    0 	   0    0 	
 24:    1    0    0    0    1    0 	   0    0 	
 25:    0    1    0    0    0    1 	   0    1 	emit: O
 26:    0    0    1    0    1    0 	   0    0 	
 27:    0    1    0    1    0    0 	   0    0 	
 28:    1    0    0    0    0    0 	   0    0 	
\end{verbatim}
  \end{small}
\end{table}
It should be noted that we search for a network with the same architecture
as the network which we constructed by hand: inputs of length $6$, outputs
of length $2$, and $4$ memory cells. This perhaps made the search easier.
However, it should also be noted that the search space has $48+8=56$ 
weights and $8+2=10$ biases. Since input weights are restricted to $3$ values
and biases to $6$ values, the total search space has
\[ 3^{56}\cdot 6^{10} \approx 3.16\cdot 10^{34}\]
nodes. Thus, our search, which terminated in minutes, had a sizeable
search space to explore (some variations led to much quicker times, in
the $10$ second range). Furthermore, repeated searches found only $2$
perfect solutions.  It is quite possible that the number of solutions
is very limited for the problem at hand, perhaps only a few.

In our solution we used a simple rule for state modification: we
simply modified a random weight or bias, by randomly choosing an
admissible value: $\{\pm1,0\}$ for weights and $\{0,1,\ldots,5\}$ for
biases. The recommended rule is to try to stay at nearly the same
energy, but for our example this did not seem to make significant
difference for the speed or quality of the solution. At some point, we
tried to tie the values of the biases to be the number of $-1$'s in
the corresponding row of the weight matrix, motivated by biases that
come out of arithmetization of formulas of boolean logic.  It turns
out that this results in significantly less successful outcome, and it
appears important that the weights and biases can be varied
independently.

\section{Continuous weights}
In the current section we allow the weights of the deductron to be
real numbers. As we can see, there is no need for $\beta$ (the inverse
temperature), as it can be easily absorbed by the weights.  Similarly,
the shift of $0.5$ used in our falling sigmoid $S$
(see~\eqref{eqn:falling-sigmoid}) can be absorbed by the biases.
Also, we choose to use the standard, rising sigmoid function:
\begin{equation}
  \label{eqn:rising-sigmoid}
  \sigma(x) = \frac{1}{1 + e^{-x}}
\end{equation}
This necessitates taking the complement of 1 when computing the
output of the net.

The loss (error) function is simply the sum of squares of errors:
\begin{equation}
  \label{eqn:loss-function}
  loss = \sum_{f=1}^{n_{frames}}\sum_{i=1}^{n_{out}} \left(t_i^{(f)}-o_i^{(f)}\right)^2. 
\end{equation}

\begin{exercise}[Gradient of loss]
  Using Figure~\ref{fig:deductron}, find the gradient of the loss
  function given by~\eqref{eqn:loss-function} over the parameters.
  That is, find the formulas for the partial derivatives:
  \begin{align*}
    \frac{\partial\, loss}{\partial\, w_{ij}^{(k)}},\\
    \frac{\partial\, loss}{\partial\, b_{i}^{(k)}},
  \end{align*}
  where $W_k = \left[w_{ij}^{(k)}\right]$ and $b_k =\left[ b_{i}^{(k)}\right]$, $k=1,2$ are
  the weight matrices and bias vectors of the deductron.
\end{exercise}

The above exercise is important when one to wants to implement a
variation of Gradient Descent in order to find optimal weights and
biases.  The mechanics of differentiation is not particularly
interesting. However, for complex neural networks it represents a
challenge when implemented by manual application of the Chain
Rule. Therefore, a technique called \textbf{automatic differentiation}
is used, which essentially implements the Chain Rule in software. The
computer manipulates the formulas expressing loss to obtain the
gradient.  The system Tensorflow \cite{tensorflow2015-whitepaper}
provides the facility to carry it out with a minimum amount of effort
and allows for quick modification of the model.  In contrast, the
human would have to essentially repeat the calculations manually for
each model variation, which inhibits experimentation.

Following the documentation of Tensorflow, we implemented training
of a deductron RNN, closely following Figure~\ref{fig:deductron}.
The implementation details are presented in Appendix~\ref{appendix:python}.

\begin{exercise}[Generating W-language samples with interval maps]
  Use the interval map in Figure~\ref{fig:chaotic-mapping} to generate
  samples of the W-language, like the one below:
  \begin{center}
    \begin{tikzpicture}[scale=0.52]
      \node {\scalebox{1}[-1] {
          \includegraphics[width=\textwidth,height=0.5cm]{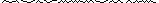}}};
    \end{tikzpicture}
  \end{center}
  Assume that the image begins and ends with exactly one blank (not
  shown). The image has exactly 155 columns.  This sample
  should decode to the following decoded message, with 'X' and 'O'
  appearing at the time of their emission:
  \begin{center}
    \begin{inact}
      ____XX__________X__________OOOO\\
      _______X________OOO___________X\\
      ___X_____X___X___X______X___XX_\\
      _______OOOO________X_____XX____\\
      ____XX______X_____X___X___X____\\
      ____
    \end{inact}
  \end{center}
  Use the samples generated with the interval map instead of the samples
  generated with a random number generator to train the Deductron
  to recognize the W-language. NOTE: You will have to slightly perturb the
  mapping of the interval, as multiplication by 2 and 4 leads to rapid decay
  of the precision on computers using base-2 arithmetic, ending in a constant
  sequence after several dozen of iterations.
\end{exercise}

\section{Conclusions}
In our paper we constructed a non-trivial and mathematically rigorus
example of a class of image data representing encoded messages which
requires long-term memory to decode. 

We constructed a conventional computer program for decoding the data.
The program was subsequently translated to a Recurrent Neural Network.
Subsequently, we generalized the neural network to a class of neural
networks, which we call \emph{deductrons}. A deductron is called that
because it is a 3-layer neural network, with a middle layer capable
of simple inferences.

Finally, we demonstrated that our neural networks can be trained by
using global optimization methods. In particular, we demonstrated that
simulated annealing discovers an algorithm which decodes the class of
inputs with 100\% accuracy, and is logically equivalent to our first
handcrafted program. We also showed how to train deductrons using
Tensorflow and Adam optimizer.

Our analysis opens up a direction of research on RNN which have more
clear semantics than other RNN, such as LSTM, with a possibility of
better introspection into the workings of the optimal programs. It is to be
determined whether our RNN is more efficient than LSTM. We conjecture
that the answer is ``yes'' and that our architecture is a class of RNN
which can be trained faster and understood better from the theoretical
standpoint.

\bibliographystyle{plain}
\bibliography{deductron_paper.bib}
\newpage
\appendix
\section{Python codes}
\label{appendix:python}
We present some programs which illustrate in detail the approaches
explained in the current paper. Many of the programs require test data
in the file named {\tt data.py}. The listing of this file is not
included in the paper due to its large length, but it accompanies the
paper as a separate file, along with all Python code listed in the
paper.

\subsection{A simple Deductron simulator}
The necessary operations to implement a deductron simulation, and the
loss function, are defined in this Python code:
\begin{small}
\begin{lstlisting}[style={python-style}]
#
# @file   deductron_base.py
# @author Marek Rychlik <marek@cannonball.lan>
# @date   Sun Jun 17 09:47:17 2018
# 
# @brief  A deductron simulator 
# 
# This simple implementation of a deductron simulator
# uses numpy. 
#

import numpy as np

class DeductronBase:  
    '''Deductron base. Does not define weights and biases.  Add suitable
    fields W1, B1, W2, B2 to make it work.

    '''
    W1 = None
    B1 = None
    W2 = None
    B2 = None
    out = None

    def __init__(self, beta = 1, shift = 0):
        ''' Sets beta and shift, which are applied to activations
        before standard (rising) sigmoid is applied. '''
        self.beta = beta
        self.shift = shift

    def sigmoid(self, x):
        ''' Implements sigmoid function. '''
        return 1 / (1 + np.exp(- self.beta * (x - self.shift)))


    def __call__(self, input):
        ''' Runs a deductron on input and returns output '''

        #
        # Run input classifier
        #
        h = self.sigmoid( self.W1 @ input + self.B1)

        #
        # Split neuron outputs into two groups
        #
        left,right = np.split(h, 2)
        n_frames = 1 if len(input.shape) == 1  else input.shape[1]
        n_memory = left.shape[0]

        #
        # Implement V-gate
        #
        prod = left * right
        z = np.zeros((n_memory, n_frames))
        for t in range(1,n_frames):
            z[:,t] = prod[:,t-1] * z[:,t-1] + (1.0 - left[:,t-1])

        #
        # Classify memories
        #
        self.out = 1.0 - self.sigmoid( self.W2 @ z + self.B2 )

        return self

    def loss(self, targets):
        '''Loss - sum of squares of errors.'''
        diff = self.out - targets
        return np.sum(np.square(diff))

    def __str__(self):
        return ("{}:\n"
                "beta: {}\n"
                "shift: {}\n"                
                "W1:\n{}\n"
                "B1:\n{}\n"
                "W2:\n{}\n"
                "B2:\n{}\n").format(self.__class__.__name__,
                                    self.beta,
                                    self.shift,
                                    self.W1,
                                    self.B1,
                                    self.W2,
                                    self.B2)


class WLangDecoderExact(DeductronBase):
    '''Implements exact decoder derived in the white paper. '''

    def __init__(self):
        super(WLangDecoderExact, self).__init__(beta = 10, shift = 0.5)

        self.W1 = np.array([
            #  0,0  1,0,    2,0,   0,1   1,1    2,1
            [  0,     1,      1,     0,    0,    -1  ], # y[0][0][*];
            [  1,     1,      0,    -1,    0,     0  ], # y[0][1][*];
            [  1,     0,      0,    -1,    0,     0  ], # y[0][2][*];
            [  0,     0,      1,     0,    0,    -1  ], # y[0][3][*];
            # 0,0  1,0,    2,0,   0,1   1,1    2,1
            [  1,     1,      0,    -1,    0,     0  ], # y[1][0][*];
            [  0,     1,      1,     0,    0,    -1  ], # y[1][1][*];
            [ -1,     0,      0,     0,    0,     0  ], # y[1][2][*];
            [  0,     0,     -1,     0,    0,     0  ], # y[1][3][*];
            ]).astype(np.float32)

        self.B1 = np.array([
            [1],   [1],    [1],   [1],
            [1],   [1],    [1],   [1],
            ]).astype(np.float32)

        self.W2 = np.array([
            #     0      1     2    3    
            [    -1,     0,   -1,    0  ],
            [     0,    -1,    0,   -1  ],
            ]).astype(np.float32)

        self.B2 = np.array([
            [2],
            [2],
            ]).astype(np.float32)



\end{lstlisting}
\end{small}
The class \emph{WLangDecoderExact} derived from \emph{DeductronBase}
defines the deductron with a particular set of weights derived in this
paper to perform exact decoding of the W-language. The following
code explains how to use the class to compute the loss on sample
inputs; in addition to the exact decoder it contains several
decoders which are a result of various ways to train the RNN:
\begin{small}
\begin{lstlisting}[style={python-style}]
#
# @file   deductron.py
# @author Marek Rychlik <marek@cannonball.lan>
# @date   Sun Jun 17 09:47:17 2018
# 
# @brief  A deductron simulator and examples
# 
# This simple implementation of a deductron simulator
# uses numpy in class DeductronBase. Subsequently
# several concrete weight/bias combinations are provided
# as found by Tensorflow with different training sets.
#

import numpy as np
from deductron_base import *
import data                     # For sample inputs

class WLangDecoderLargeModel1(DeductronBase):  
    ''' Deductron trained on a 'stretched' sample of length 518. '''
    def __init__(self):
        super(WLangDecoderLargeModel1, self).__init__(beta = 10, shift = 0.5)
        self.W1 = np.array([
            [ 1,  1, -1, -1,  1,  1],
            [ 0,  0, -1,  1, -1,  1],
            [-1,  1, -1,  1, -1,  1],
            [-1, -1, -1, -1, -1, -1],
            [ 1, -1, -1,  1,  0, -1],
            [ 0, -1,  1,  0, -1,  0]
            ]).astype(np.float32)

        self.B1 = np.array([
            [2], [1], [0], [3], [1], [0]
            ]).astype(np.float32)

        self.W2 = np.array([
            [ 1, -1,  1],
            [-1,  1,  1]]).astype(np.float32)
        self.B2 = np.array([[1], [1]]).astype(np.float32)
    
class WLangDecoderCombModel1(deductron_base.DeductronBase):  
    '''Deductron trained on combined 'small' and 'stretched' samples of
    total size 29 + 518.

    '''
    def __init__(self):
        super(WLangDecoderCombModel1, self).__init__(beta = 10, shift = 0.5)
        self.W1 = np.array([
            [ 0,  1,  1,  1,  1, -1],
            [-1,  0,  1,  1, -1,  0],
            [ 0, -1,  0, -1,  0, -1],
            [ 1,  1, -1, -1,  0,  1],
            [-1, -1,  0,  0,  0,  1],
            [-1, -1, -1, -1, -1, -1]]).astype(np.float32)
        self.B1 = np.array([[1], [0], [2], [2], [0], [0]]).astype(np.float32)
        self.W2 = np.array([
            [-1,  1, -1],
            [ 1, -1, -1]]).astype(np.float32)
        self.B2 = np.array([[2], [2]]).astype(np.float32)

class WLangDecoderCombModel2(deductron_base.DeductronBase):  
    '''Deductron trained on combined 'small' and 'stretched' samples of
    total size 29 + 518. Obtained using Tensorflow.

    '''
    def __init__(self):
        super(WLangDecoderCombModel2, self).__init__(beta = 1, shift = 0)
        self.W1 = np.array([[ 3.73, -1.17,  2.82, -0.77,  1.71,  4.,  ],
                            [-0.01, -2.16,  4.53,  4.85, -2.14,  2.66 ],
                            [-2.11,  3.37, -1.81,  0.97, -4.08,  1.7  ],
                            [ 4.98,  4.58, -4.49, -4.28,  4.43,  4.93 ],
                            [ 2.88, -1.26, -1.01, -1.42, -0.02,  2.09 ],
                            [ 0.22, -2.01,  0.05, -0.65,  0.12, -1.11 ]]
       ).astype(np.float32)
        self.B1 = np.array([[ 3.66],
                            [ 2.45],
                                [-1.54],
                                [ 5.29],
                                [ 3.45],
                                [-1.63]]).astype(np.float32)
        self.W2 = np.array([
            [-55.81,  47.05,  15.44],
            [ 29.19, -32.,    24.45]]).astype(np.float32)
        self.B2 = np.array([[ 3.42], [12.51]]).astype(np.float32)

#
# Sliding windows for string 'XO'
# surrounded by blanks
#
tiny_inputs = np.array([
    [0, 0, 0, 0, 0, 1],
    [0, 0, 1, 0, 1, 0],
    [0, 1, 0, 1, 0, 0],
    [1, 0, 0, 0, 1, 0],
    [0, 1, 0, 0, 0, 1],
    [0, 0, 1, 0, 0, 0],
    [0, 0, 0, 1, 0, 0],
    [1, 0, 0, 0, 1, 0],
    [0, 1, 0, 0, 0, 1],
    [0, 0, 1, 0, 1, 0],
    [0, 1, 0, 1, 0, 0],
    [1, 0, 0, 0, 0, 0],            
    ]).transpose()

tiny_targets = np.array([
    [0, 0],
    [0, 0],
    [0, 0],
    [1, 0],
    [0, 0],
    [0, 0],
    [0, 0],
    [0, 0],
    [0, 0],
    [0, 1],
    [0, 0],
    [0, 0],
    ]).transpose()

# Choose one of the neural nets
#nn = WLangDecoderExact()        # Definitely always works
#nn = WLangDecoderLargeModel1()   # Only works on large sample
#nn = WLangDecoderCombModel1()    # Works except for tiny sample
nn = WLangDecoderCombModel2()    # Works for all

n_digits = 4
print()
print("Tiny sample loss:   ", nn(tiny_inputs).loss(
    tiny_targets
    ).round(n_digits))
print("Small sample loss:  ", nn(data.small_inputs).loss(
    data.small_targets
    ).round(n_digits))
print("Large sample loss:  ", nn(data.large_inputs).loss(
    data.large_targets
    ).round(n_digits))
print("Combined sample loss", nn(data.comb_inputs).loss(
    data.comb_targets
    ).round(n_digits))
\end{lstlisting}
\end{small}
\subsection{Training implemented in Python}
This is a ``pure'' Python implementation, utilizing only \emph{numpy}.
The training algorithm is a version of simulated annealing.  The
weights and biases are quanitized as described in the current paper.
Most of the code is devoted to picking a neighbor of the
deductron obtained by choosing one of the weights or biases, and
replacing it with a randomly chosen admissible value. The algorithm
keeps track of the best state found so far. If it gets stuck not
finding a lower energy state for a long time, it restarts with
the best state so far, thus implementing a form of backtracking.

The following code implements simulated annealing as a training
method for the deductron, and a simple function \emph{test\_annealing}
which trains the network on given training data, used to
obtain several sample weight/bias combinations listed in this paper.
\begin{small}
\begin{lstlisting}[style={python-style}]
import deductron_base
import numpy as np
import copy
import random

class QuantizedDeductron(deductron_base.DeductronBase):

    ADMISSIBLE_WEIGHTS = [-1,0,1]
    ADMISSIBLE_BIASES  = [0,1,2,3,4,5]

    REP_MAX = 8192
    BETA_MAX = 10

    def __init__(self, n_in, n_memory, n_out, beta = 1):
        super(QuantizedDeductron, self).__init__(beta = beta, shift = 0.5)
        self.__create_random_weights(n_in, n_memory, n_out)

    def __create_random_weights(self, n_in, n_memory, n_out):
        self.W1 = np.reshape(
            random.choices(QuantizedDeductron.ADMISSIBLE_WEIGHTS,
                           k = n_in * 2 *n_memory), (2*n_memory, n_in))
        self.B1 = np.reshape(
            random.choices(QuantizedDeductron.ADMISSIBLE_BIASES,
                           k = 2 *n_memory), (2*n_memory, 1))        
        self.W2 = np.reshape(
            random.choices(QuantizedDeductron.ADMISSIBLE_WEIGHTS,
                           k = n_memory * n_out), (n_out, n_memory))        
        self.B2 = np.reshape(
            random.choices(
                QuantizedDeductron.ADMISSIBLE_BIASES,
                k = n_out),
            (n_out, 1))

    def _modify_W1(self):
        m,n = self.W1.shape
        i = random.choice(range(m))
        j = random.choice(range(n))
        old = self.W1[i,j]
        self.W1[i,j] = random.choice(
            QuantizedDeductron.ADMISSIBLE_WEIGHTS)

        return ('W1', (i, j), old)

    def _modify_B1(self):
        i = random.choice(range(self.B1.shape[0]))
        old = self.B1[i,0]
        self.B1[i] = random.choice(
            QuantizedDeductron.ADMISSIBLE_BIASES) - 0.5

        return ('B1', (i), old)

    def _modify_W2(self):
        m,n = self.W2.shape
        i = random.choice(range(m))
        j = random.choice(range(n))
        old = self.W2[i,j]
        self.W2[i,j] = random.choice(
            QuantizedDeductron.ADMISSIBLE_WEIGHTS)

        return ('W2', (i, j) , old)

    def _modify_B2(self):
        i = random.choice(range(self.B2.shape[0]))
        old = self.B2[i,0]
        self.B2[i] = random.choice(
            QuantizedDeductron.ADMISSIBLE_BIASES) - 0.5

        return ('B2', (i), old)

    def modify(self):
        i = random.choice(range(4))
        if i == 0:
            return self._modify_W1()
        elif i == 1:
            return self._modify_B1()            
        elif i == 2:
            return self._modify_W2()            
        elif i == 3:
            return self._modify_B2()            

    def restore(self, mod_data):
        name, idx, old = mod_data
        if name == 'W1':
            i, j = idx
            self.W1[i,j] = old
        elif name == 'B1':
            i = idx
            self.B1[i] = old
        elif name == 'W2':
            i, j = idx
            self.W2[i,j] = old
        elif name == 'B2':
            i = idx
            self.B2[i] = old

    def run_loss(self, inputs, targets):
        return self(inputs).loss(targets)
            
    @staticmethod
    def train(n_memory, inputs, targets, beta_incr = 0.001):
        n_in, _ = inputs.shape
        n_out, _ = targets.shape        

        print("**** Simulated annealing ****")
        net = QuantizedDeductron(n_in, n_memory, n_out, beta = 0)
        E0 = net.run_loss(inputs, targets)
        net_best = copy.deepcopy(net)
        E_best = E0
        print("%10s %10s %10s %10s"
              % ("Iteration", "Loss", "Best Loss", "Inv. Temp."))
        iter = 0
        rep = 0
        while net.beta < QuantizedDeductron.BETA_MAX:
            mod_data = net.modify()
            E1 = net.run_loss(inputs, targets)
            prob = np.exp( -net.beta * (E1 - E0) );
            if  random.uniform(0, 1) < prob:
                rep = 0
                E0 = E1         # Accept
                if  ( iter % 1000 == 0) or ( E0 < E_best ):
                    print("%10.3d %10.3f %10.3f %10.3f"
                          % (iter, E0, E_best, net.beta))
                    
                if E0 < E_best:
                    E_best = E0
                    net_best = copy.deepcopy(net)
            else:
                # Reject, restore weight, update
                rep += 1
                if (rep < 1000 and rep % 100 == 0) or (rep % 1000 == 0):
                    print("Repeats:", rep)

                if rep >= QuantizedDeductron.REP_MAX:
                    print ("Restart on iteration", iter,
                               "repetitions:", rep) 
                    # Restart
                    net = net_best
                    rep = 0
                else:
                    net.restore(mod_data)
            iter += 1
            net.beta += beta_incr
        return (net_best, E_best)

import data

def test_annealing(inputs, targets):
    net, loss = QuantizedDeductron.train(
        3, inputs, targets, beta_incr = 0.00001)
    print(str(net))
    return (net, loss)

\end{lstlisting}
\end{small}
The output is illustrated in Figure~\ref{fig:learned-weights}.

\subsection{Training implemented in Tensorflow}
The following Python/Tensorflow \cite{tensorflow2015-whitepaper} code
provides training to the deductron.
\begin{small}
\begin{lstlisting}[style={python-style}]
#
# @file   deductron_train.py
# @author Marek Rychlik <marek@cannonball.lan>
# @date   Sun Jun 17 10:25:51 2018
# 
# @brief  Implements Deductron training using Tensorflow
# 
#
#

from __future__ import absolute_import, division, print_function

import tensorflow as tf
import numpy as np
import matplotlib.pyplot as plt
import io
from datetime import datetime
import data                     # Where the training data are

#tf.enable_eager_execution() 
#tf.executing_eagerly()        # => True

# Get training data; should be numpy 2D arrays
#_inputs  = data.small_inputs
#_targets = data.small_targets

_inputs  = data.comb_inputs
_targets = data.comb_targets

# Set parameters for the network
n_memory   = 3                  # Num. of memory cells
input_len  = 6                  # Input frame size
output_len = 2                  # Output frame size

assert(_inputs.shape[0] == input_len)
assert(_targets.shape[0] == output_len)

################################################################
#
#   Define the graph
#
################################################################

with tf.name_scope("classify_inputs"):
    inputs  = tf.constant(_inputs,  tf.float32, name="inputs")
    targets = tf.constant(_targets, tf.float32, name="targets")
    n_frames = inputs.shape[1]
    W1 = tf.get_variable("W1", shape = [2*n_memory, input_len],
                         dtype = tf.float32)
    B1 = tf.get_variable("B1", shape = [2*n_memory, 1],
                         dtype = tf.float32)
    # Inverse temperature
    h = tf.sigmoid(  tf.matmul(W1, inputs) + B1 )
    [left,right] = tf.split(h, num_or_size_splits=2, axis=0)
    prod = tf.multiply(left, right) #Hadamard product
    prod = tf.unstack(prod, axis=1)
    left = tf.unstack(left, axis=1)

with tf.name_scope("memory"):
    u = tf.zeros([n_memory])
    zlst = [u]
    for t in range(1, n_frames):
        ## Memory
        u = prod[t-1] * u + ( 1.0 - left[t-1])
        zlst.append(u)
        z = tf.stack(zlst, axis = 1)

with tf.name_scope("output"):
    W2 = tf.get_variable("W2", shape = [output_len, n_memory],
                         dtype = tf.float32)
    B2 = tf.get_variable("B2", shape = [output_len, 1],
                         dtype = tf.float32)
    out = 1.0 - tf.sigmoid( tf.matmul(W2, z ) + B2)
    #loss = -tf.reduce_mean(tf.log(out) * targets
    #                       + tf.log(1.0 - out) * (1.0 - targets))
    diff = out-targets;
    loss1 = tf.reduce_sum(tf.square(diff))
    loss2 = tf.reduce_sum(tf.square(W1))
    loss3 = tf.reduce_sum(tf.square(B1))
    loss4 = tf.reduce_sum(tf.square(W2))
    loss5 = tf.reduce_sum(tf.square(B2))
    eps1 = 0.0001; eps2 = 0.00001
    loss = loss1  + eps1 * (loss2 + loss3) + eps2 * (loss4 + loss5)
    
    
################################################################
#
#   Run the training
#
################################################################

# Create an optimizer with the desired parameters.
#opt = tf.train.GradientDescentOptimizer(learning_rate=0.5)
opt = tf.train.AdamOptimizer(learning_rate=0.001)
#opt = tf.train.AdagradOptimizer(learning_rate=0.5)

#
# Run the training
#
train = opt.minimize(loss)
n_steps = 60000                 # Number of steps
loss_threshold = 0.01           # Stop if loss < this value

date = datetime.now().isoformat() # Label for summaries
tf.summary.scalar("loss-" + date, loss)
write_op = tf.summary.merge_all()

with tf.Session() as sess:
    writer = tf.summary.FileWriter('logs', sess.graph)
    init = tf.global_variables_initializer()
    sess.run(init)

    for step in range(n_steps):
        #grads_and_vars = opt.compute_gradients(loss)
        #new_grads_and_vars = grads_and_vars
        #train = opt.apply_gradients(grads_and_vars)
        #sess.run(step.assign_add(1))
        sess.run(train)
        loss_value = sess.run(loss)
        loss1_value = sess.run(loss1)
        if step % 100 == 0:
            print("Step: ", step,
                  "Loss: ", loss_value,
                      "Real Loss: ", loss1_value)

        summary = sess.run(write_op)
        writer.add_summary(summary, step)
        writer.flush()
        if loss1_value < loss_threshold:
            print('*** Iteration stopped when loss < ',
                  loss1_value)
            break

    writer.close()

    print("W1:\n{}".format(np.round(W1.eval(),2)))
    print("B1:\n{}".format(np.round(B1.eval(),2)))
    print("W2:\n{}".format(np.round(W2.eval(),2)))
    print("B2:\n{}".format(np.round(B2.eval(),2)))
    print("Outputs: {}".format(np.round(np.transpose(out.eval()),1)))    
    print("Loss: {}".format(loss.eval()))
    print("Real Loss: {}".format(loss1.eval()))

    sess.close()

\end{lstlisting}
\end{small}
We instrumented the code with logging \emph{summaries}, which can be
used to visualize learning in a standard browser progress using
Tensorboard (a log-viewing program which is typically distributed
with Tensorflow).


\newpage

\section{Peephole LSTM Architecture}
\label{appendix:LSTM}
The flow of data in an LSTM is illustrated by
the following diagram (\cite{wiki:xxx}):
\begin{center}
  \includegraphics[width=.7\linewidth]{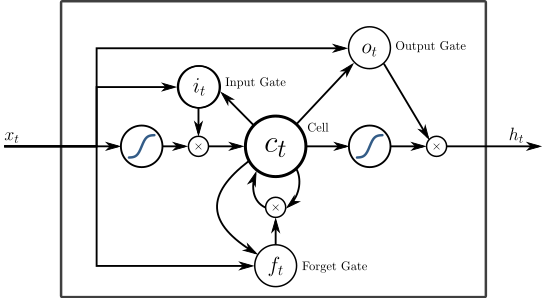}
\end{center}
The formulas expressing the data transformations are:
\begin{align}
  f_t &= \sigma_g(W_{f} x_t + U_{f} c_{t-1} + b_f) \\
  i_t &= \sigma_g(W_{i} x_t + U_{i} c_{t-1} + b_i) \\
  o_t &= \sigma_g(W_{o} x_t + U_{o} c_{t-1} + b_o) \\
  c_t &= f_t \circ c_{t-1} + i_t \circ \sigma_c(W_{c} x_t + b_c) \\
  h_t &= o_t \circ \sigma_h(c_t)
\end{align}
Every quantity is a vector. The symbol ``$\circ$'' stands for the
Hadamard (elementwise) product.  Thus, to perform the product, the
vectors have to have the same length.
The functions $\sigma_{*}$ are sigmoid activation functions.

\newpage
\newpage

\end{document}